\documentclass[journal]{IEEEtran}
\usepackage{epsfig}
\usepackage{amsmath}
\usepackage{amssymb}

\usepackage{color}
\usepackage{bm}
\usepackage{algorithm}
\usepackage{algorithmic}
\usepackage{multirow}
\usepackage{subfigure}
\usepackage{amsthm}
\usepackage{booktabs}
\usepackage{textcomp}
\usepackage{mathrsfs}
\usepackage{makecell}
\usepackage{diagbox}
\usepackage{array}

\usepackage{arydshln}
\allowdisplaybreaks[4]
\usepackage{subeqnarray}
\usepackage{cases}

\usepackage{bbding}
\usepackage{pifont}
\usepackage{wasysym}
\usepackage{amssymb}

\newtheorem{thm}{Theorem}
\newtheorem{cor}{Corollary}
\newtheorem{prop}{Proposition}
\newtheorem{assum}{Assumption}

\newtheorem{remark}{Remark}

\usepackage{threeparttable}

\def\x{\mathbf{x}}
\def\y{\mathbf{y}}

\def\c{\mathbf{c}}

\def\F{\mathcal{F}}

\begin{document}
%
% paper title
% Titles are generally capitalized except for words such as a, an, and, as,
% at, but, by, for, in, nor, of, on, or, the, to and up, which are usually
% not capitalized unless they are the first or last word of the title.
% Linebreaks \\ can be used within to get better formatting as desired.
% Do not put math or special symbols in the title.
\title{Investigating Customization Strategies and Convergence Behaviors of Task-specific ADMM}

\author{Risheng Liu,~\IEEEmembership{Member,~IEEE,} Pan Mu, Jin Zhang
	%        Pan Mu, Xiaoming Yuan, Hangrui Yue, Jin Zhang
	\IEEEcompsocitemizethanks{\IEEEcompsocthanksitem R.Liu is with the DUT-RU International School of Information Science $\&$ Engineering, Dalian University of Technology, and also with the Key Laboratory for Ubiquitous Network and Service Software of Liaoning Province, Dalian 116024, China. E-mail: rsliu@dlut.edu.cn.\protect
		\IEEEcompsocthanksitem Pan Mu is with the School of Mathematical Sciences, Dalian University of Technology, and also with the Key Laboratory for Ubiquitous Network and Service Software of Liaoning Province. (muyifan11@mail.dlut.edu.cn)\protect
		\IEEEcompsocthanksitem Jin Zhang is with the Department of Mathematics, SUSTech International Center for Mathematics, Southern University of Science and Technology, Shenzhen, China. (Corresponding author, zhangj9@sustech.edu.cn)\protect}
	%\thanks{Manuscript received April 19, 2005; revised August 26, 2015.}
}
%%%%%%%%%%%%%%%%%%%%%%%%%%%%%%%%%%%%%%%%%%%%%%%%%%%%%

% The paper headers
\markboth{Journal of \LaTeX\ Class Files,~Vol.~14, No.~8, August~2015}%
{Shell \MakeLowercase{\textit{et al.}}: Bare Demo of IEEEtran.cls for IEEE Journals}
% The only time the second header will appear is for the odd numbered pages
% after the title page when using the twoside option.
% 
% *** Note that you probably will NOT want to include the author's ***
% *** name in the headers of peer review papers.                   ***
% You can use \ifCLASSOPTIONpeerreview for conditional compilation here if
% you desire.

% If you want to put a publisher's ID mark on the page you can do it like
% this:
%\IEEEpubid{0000--0000/00\$00.00~\copyright~2015 IEEE}
% Remember, if you use this you must call \IEEEpubidadjcol in the second
% column for its text to clear the IEEEpubid mark.

% use for special paper notices
%\IEEEspecialpapernotice{(Invited Paper)}

% make the title area
\maketitle

\begin{abstract}
Alternating Direction Method of Multiplier (ADMM) has been a popular algorithmic framework for separable optimization problems with linear constraints. For numerical ADMM fail to exploit the particular structure of the problem at hand nor the input data information, leveraging task-specific modules (e.g., neural networks and other data-driven architectures) to extend ADMM is a significant but challenging task. This work focuses on designing a flexible algorithmic framework to incorporate various task-specific modules (with no additional constraints) to improve the performance of ADMM in real-world applications. Specifically, we propose Guidance from Optimality (GO), a new customization strategy, to embed task-specific modules into ADMM (GO-ADMM). By introducing an optimality-based criterion to guide the propagation, GO-ADMM establishes an updating scheme agnostic to the choice of additional modules. The existing task-specific methods just plug their task-specific modules into the numerical iterations in a straightforward manner. Even with some restrictive constraints on the plug-in modules, they can only obtain some relatively weaker convergence properties for the resulted ADMM iterations. Fortunately, without any restrictions on the embedded modules, we prove the convergence of GO-ADMM regarding objective values and constraint violations, and derive the worst-case convergence rate measured by iteration complexity. Extensive experiments are conducted to verify the theoretical results and demonstrate the efficiency of GO-ADMM. 

\end{abstract}

% Note that keywords are not normally used for peerreview papers.
\begin{IEEEkeywords}
Task-specific ADMM, guidance from optimality, convergence behaviors analysis, computer vision applications. 
\end{IEEEkeywords}

% For peer review papers, you can put extra information on the cover
% page as needed:
% \ifCLASSOPTIONpeerreview
% \begin{center} \bfseries EDICS Category: 3-BBND \end{center}
% \fi
%
% For peerreview papers, this IEEEtran command inserts a page break and
% creates the second title. It will be ignored for other modes.
\IEEEpeerreviewmaketitle

\section{Introduction}\label{sec:intro}

%\section{Introduction}
\IEEEPARstart{A}{broad} spectrum of real-world applications, ranging from image processing \cite{chan2011alternating,ng2011fast} to compressive sensing \cite{yang2011alternating,liu2019convergence}, subspace clustering~\cite{liu2013robust,elhamifar2013sparse} and machine learning \cite{liu2013linearized,sun2016deep}, can be (re)formulated as the following separable optimization model with linear constraint:
\begin{equation}\label{eq:model}
\begin{array}{c}
\min\limits_{\mathbf{x}\in\mathbb{R}^n,\mathbf{y}\in\mathbb{R}^m} f(\mathbf{x}) + g(\mathbf{y}),\ \text{s.t.}\ \mathbf{A}\mathbf{x}+\mathbf{B}\mathbf{y} =\mathbf{c},
\end{array}
\end{equation}
where $f:\mathbb{R}^n\rightarrow\mathbb{R}$ and $g:\mathbb{R}^m\rightarrow\mathbb{R}$ are closed, proper and convex functions, $\mathbf{A}\in\mathbb{R}^{l\times n}$, $\mathbf{B}\in\mathbb{R}^{l\times m}$ and $\mathbf{c}\in\mathbb{R}^l$. In computer vision and learning scenarios, the fidelity term $f$ captures the loss of data fitting, which is further assumed to be continuously differentiable; the regularization/prior term $g$ usually is nonsmooth which promotes desired distribution on the solution. To solve Eq.~\eqref{eq:model}, the Alternating Direction Method of Multipliers (ADMM) proposed in~\cite{glowinski1975approximation} becomes a benchmark solver because of its features of easy implementability, competitive numerical performance and wide applicability in various areas.

The ADMM has been receiving attention from a broad spectrum of areas, and various variants have been well studied in the literature. When applying original ADMM to solve Eq.~\eqref{eq:model} arising from data science applications, how to efficiently solve subproblems plays an important role in the algorithm implementation. To address this issue, plenty of numerical variants have been investigated to tackle subproblems, such as proximal and Linearized ADMM (LADMM)~\cite{liu2013linearized,yang2013linearized,xie2019differentiable}, Inexact ADMM (IADMM)~\cite{yuan2005improvement,eckstein2018relative,yue2018implementing}. Particularly, in image processing applications, after linearizing the quadratic term, the nonsmooth $\y$-subproblem admits a closed form solution. More detailed discussions can be found in survey papers~\cite{boyd2011distributed,glowinski2014alternating}. 
Unfortunately, even with the proximal or linearizing techniques, for real-world tasks in learning and vision problems, the direct application of these ADMM variants usually leads to performance far from satisfactory. This is because the underlying structure at hand and the input data information have not been well exploited. In fact, these information could help the optimization model find the task-related solution. This motivates us to introduce additional task-specific modules, such as designed filters or trained network architectures, to customize ADMM scheme to address specific applications.

\begin{table*}[htb!]
	\centering
	\renewcommand\arraystretch{1.3}
	\caption{ Summarizing of the essential comparison aspects of representative methods. The first column represents the classification of these methods, i.e., numerical variants (Num., e.g., ADMM~\cite{he2015non}) and specific task variants (Task., e.g., PP-ADMM~\cite{chan2017plug}, PnP~\cite{ryu2019plug}) and RED~\cite{romano2017little}. }
	\label{tab:inexact}
		\begin{small}
			\begin{threeparttable}
			\begin{tabular}{| c | c | c | c | c | c | c | c |} 
				\hline
				Types & Methods & Cond. & Fide. & Reg. & Learn. & Opt. &  Rate \\
				\hline
				\hline
				Num. & \cite{he2015non} & $f$ and $g$ are convex & \ding{51}  & \ding{51}  & \ding{55} & \ding{51} & \ding{51} \\
				\hline
				\multirow{3}{*}{Task.} & \cite{chan2017plug} &$\nabla_{\mathbf{x}}f$ is bounded, $\mathcal{D}$ is nonexpansive &  \ding{51}  &  \ding{55} & \ding{51}  & \ding{55} & \ding{55} \\
				\cline{2-8}
				& \cite{ryu2019plug} & $f$ is strongly convex, $\mathcal{D}$ is nonexpansive  &  \ding{51}   & \ding{55} &\ding{51} & \ding{55} &\ding{55} \\
				\cline{2-8}
				& \cite{romano2017little} & $f$ and $g$ are convex, $\|\nabla_{\mathbf{x}}D(\mathbf{x})\|\leq C$  & \ding{51}  &\ding{51} &\ding{51} & \ding{55} & \ding{55} \\
				\hline
				Ours & GO-ADMM & $f$ and $g$ are convex & \ding{51} & \ding{51} & \ding{51} & \ding{51} & \ding{51}\\
				\hline
			\end{tabular}
%			\begin{itemize}
			\begin{tablenotes}
			\footnotesize 
			\item  ``Cond.'' denotes conditions required about $f$, $g$ and $\mathcal{D}$. 
			``Fide.'' represents that the data fidelity term holds. 
			``Reg.'' denotes preserving the regularization.
			``Learn.'' implies inserting learning-based information.
			``Opt.'' means iteration sequence converges to the optimal solution of objective.
			``Rate'' represents that convergence rate can be achieved.
		\end{tablenotes}
		\end{threeparttable}
		\end{small}
\end{table*}

Among the ADMMs with embedded task-specific modules, the most prevailing class is introduced in ~\cite{venkatakrishnan2013plug}, named Plug-and-Play ADMM (PP-ADMM). The PP-ADMM has gained popularity in a broad spectrum of areas, such as super-resolution~\cite{zhang2019deep,chan2017plug}, image denoising~\cite{dong2018denoising,buzzard2018plug,rond2016poisson}, inpainting~\cite{tirer2018image} and image reconstruction~\cite{ryu2019plug,liu2019converged}. The plug-and-play scheme allows one to completely substitute the prior subproblem by manually designed computations within the ADMM scheme. Extremely promising performance has been widely witnessed in image restoration and signal recovery tasks. A variety of task-specific modules has been used for plugging in ADMM framework, ranging from deep-learning-based denoisers~\cite{chan2017plug,sun2019block,meinhardt2017learning,liu2020free}, Gaussian Mixture Model (GMM)~\cite{li2016rain,teodoro2018convergent} to designed filters~\cite{venkatakrishnan2013plug}, etc. However, replacing the prior term with implicit denoising modules usually leads to an unclear definition of the objective function. Consequently, these schemes may fail to guarantee desirable solution qualities. Different from the schemes with implicit priors, Regularization by Denoising (RED)~\cite{romano2017little} adopts explicit nonconvex regularization with noise evaluation, making the overall objective function clearer and better defined. However, the validity of RED is justified only for denoisers with symmetric Jacobians. Thus, unfortunately, many state-of-the-art methods, such as BM3D~\cite{dabov2007image}, RF~\cite{gastal2011domain}, CSF~\cite{schmidt2014shrinkage}, CNN~\cite{zhang2017learning} cannot be covered; this issue has been discussed in~\cite{reehorst2018regularization}. Recently, deep learning-based methods have become state-of-the-art. Douglas-Rachford Network (Dr-Net)~\cite{aljadaany2019douglas} and learning deep CNN denoiser prior~\cite{zhang2017learning} apply ConvNets modeling data fidelity and/or image prior proximal operators. Like these learning-based schemes, aiming at obtaining data-specific iteration schemes, ADMMNet~\cite{sun2016deep,yang2017admm} introduced hyperparameters into the classical numerical solvers and then performed discriminative learning on collected training data. The existing learning-based methods perform better on vision tasks than many state-of-the-art methodologies. Due to the severe inconstancy of parameters during iterations, rigorous analysis of the resulted trajectories is also missing, which leads to the lack of strict theoretical investigations. Methods based on optimal conditions~\cite{liu2018proximal,wang2016linearized,liu2019convergence} reduced the gap between deep networks and optimization models by introducing error control condition. For example, in~\cite{liu2018proximal}, the implicit ADMM scheme converges to a fixed point while without knowing whether it is an optimal solution to the objectives. For a clear impression, we summarize the essential comparison aspects in Tab.~\ref{tab:inexact}.

\subsection{Our Contributions}
	Numerical ADMM cannot exploit the particular structure of the problem at hand nor the input data information and thus may fail in learning and visual problems. The existing task-specific schemes plug their task-specific modules into the numerical iterations in a straightforward manner. Even with some restrictive constraints on the plug-in modules, they can only obtain some relatively weaker convergence properties for the resulted ADMM iterations. Thus, leveraging task-specific modules (e.g., neural networks and other data-driven architectures) to extend ADMM is an important but challenging problem. To partially address issues in these existing task-specific methods, this work constructed a completely new algorithmic framework to design task-specific ADMM. In this paper, we propose the Guidance from Optimality (GO)-ADMM, which incorporates a mechanism to embed task-specific modules into the fidelity subproblem. The new paradigm aggregates both data fidelity and ad-hoc modules, and the prior information is reserved during iteration processes. Theoretical convergence in previous works (e.g.,~\cite{chan2017plug,chan2019performance,ryu2019plug}) often requires certain architecture constraints on the embedded modules (e.g., boundedness, nonexpansive or Lipschitz conditions). Unfortunately, verification of these assumptions is too ambitious, especially when the task-specific modules are explicitly complex deep network architectures. One notable feature of the proposed GO-ADMM relies on the fact that a first-order-optimality-based guiding policy is involved during each iteration. Thanks to this feature, the convergence of GO-ADMM is independent of any architecture constraint. This new perspective enables us to automatically identify reliable modules to build our task-specific iterations. We further rigorously prove the convergence of GO-ADMM towards a solution. Besides, we analyze the convergence rate in terms of iteration complexity. To our best knowledge, GO-ADMM seems to be the first theoretically convergent task-specific ADMM, whose convergence, surprisingly as good as those well-designed numerical ADMMs, requires no additional assumptions on embedded modules. The main contributions of this work are summarized as:

\begin{itemize}
	\item A striking feature of GO-ADMM, which differs from previous approaches is that a new mechanism is incorporated to embed modules into the fidelity subproblem. Thanks to this design, we aggregate data fidelity and ad-hoc modules, with prior information kept in reserve for our task-specific iterations. 
	\item Different from existing schemes that usually require certain architecture constraints on the additional modules, GO-ADMM adopts an optimality-based guiding policy to automatically identify reliable modules, resulting in a self-controllable propagation scheme.
	\item We strictly prove the global convergence towards a solution with quality and estimate the exact convergence rate. Our theoretical results indeed inherit analytic properties of well-designed numerical ADMM, and thus are more convincing than those existing heuristic task-specific methods.
	\item Extensive experiments verify our theories. In particular, we can observe performance even better than some state-of-the-art deep learning methods. Furthermore, our GO strategy could be extended to other first-order schemes, achieving rigorous convergence results.
\end{itemize}

\section{Review on Existing ADMMs}
Specifically, by introducing the augmented Lagrangian function of Eq.~\eqref{eq:model} with multiplier $\bm{\bm{\lambda }}$ and penalty parameter $\beta\!>\!0$
\begin{equation*}\label{eq:lagrange}
\begin{array}{r}
\mathcal{L}_{\beta}(\mathbf{x},\mathbf{y},\bm{\bm{\lambda }})=f(\mathbf{x}) + g(\mathbf{y}) -\bm{\bm{\lambda }}^{\top}(\mathbf{A}\mathbf{x}+\mathbf{B}\mathbf{y}-\mathbf{c})\\
+ \frac{\beta}{2}\|\mathbf{A}\mathbf{x}+\mathbf{B}\mathbf{y}-\mathbf{c}\|_2^2,
\end{array}
\end{equation*}
the standard ADMM scheme reads as 
\begin{align}
\mathbf{x}^{k+1} &= \arg\min\limits_{\mathbf{x}} \mathcal{L}_{\beta}(\mathbf{x},\mathbf{y}^k,\bm{\bm{\lambda }}^k),\label{eq:x-iter}\\
\mathbf{y}^{k+1} &= \arg\min\limits_{\mathbf{y}}\mathcal{L}_{\beta}(\mathbf{x}^{k+1},\mathbf{y},\bm{\bm{\lambda }}^k),\label{eq:y-iter}\\
\bm{\bm{\lambda }}^{k+1} &= \bm{\bm{\lambda }}^k - \beta\left(\mathbf{A}\mathbf{x}^{k+1}+\mathbf{B}\mathbf{y}^{k+1}-\mathbf{c}\right),\label{eq:lambda-iter}
\end{align}
which can be understood as iteratively performing the alternating minimization for the primal variables ($\mathbf{x}$, $\mathbf{y}$) and gradient ascent for the dual variable $\bm{\lambda}$. Typically, $f$ is smooth and $g$ is nonsmooth.

\textbf{Task-specific Schemes:} 
The implementations of a variant of task-specific ADMM algorithms do not need a specified regularization/prior term $g(\mathbf{y})$. Instead, they adopt an off-the-shelf image denoising module to replace $\mathbf{y}$-subproblem. A wide range of task-specific modules have been plugged into the ADMM framework, such as CSF~\cite{schmidt2014shrinkage}, BM3D~\cite{dabov2007image,chan2017plug}, GMM~\cite{teodoro2018convergent,li2016rain}, Non-Local Means (NLM)~\cite{chan2019performance,sreehari2016plug}, deep learning based denoiser~\cite{zhang2017learning,liu2018toward} and denoising autoencoders~\cite{bigdeli2017deep}, and etc. According to the characteristics of optimization models, we roughly separate them into two parts: explicit task-specific schemes and implicit task-specific schemes.

Implicit task-specific schemes build on the use of an implicit prior, which regularizes general inverse problems. For example, the plug-and-play ADMM (firstly introduced in~\cite{venkatakrishnan2013plug}) replaces the $\mathbf{y}$-subproblem by an implicit off-the-shelf algorithm. Furthermore, the alternating minimization for the primal variables ($\mathbf{x}$, $\mathbf{y}$) can be regarded as two independent numerical modules, i.e., one for implementing a simplified reconstruction operator and the other for performing a denoising operator (such as~\cite{aljadaany2019douglas}). The modified scheme by incorporating a continuation form has been introduced in~\cite{chan2017plug,sun2019online,chan2019performance,meinhardt2017learning,ryu2019plug}. In particular, the $\mathbf{y}$ iteration step is updated by using an off-the-shelf nonexpansive denoising operator. However, the objective function is not clearly defined if arbitrary denoising engines are used. Consequently, this may lead to undesirable convergence results.

Explicit schemes explicitly leverage the natural image distribution as prior. In this category, RED~\cite{romano2017little,reehorst2018regularization}, Block Coordinate RED (BC-RED)~\cite{sun2019block}, and deep RED~\cite{mataev2019deepred} have gained much significance in recent years. Explicit schemes rely on a general structured smoothness penalty term to regularize the desired inverse problem with a constructed denoising module. Nevertheless, these explicit regularization forms lack supervision for the iteration steps. Moreover, the existing convergent results are less convincing.

Learning data-fidelity term has drawn much attention recently~\cite{dong2018learning,aljadaany2019douglas}. For example, in~\cite{dong2018learning}, a discriminative framework is used to learn the data term (i.e., $\mathbf{x}$-subproblem) in a cascaded manner. In~\cite{aljadaany2019douglas}, Dr-Net framework is developed to replace the proximal operator in both the data fidelity term and the prior term with two different networks that firmly satisfy the non-expansive condition. A drawback of these learning data-fidelity frameworks is the lacking of theoretical guarantees.

In summary, different from existing numerical variants of ADMM in which usually show weak performance in real-world applications, the proposed GO-ADMM allows any off-the-shelf modules of imaging systems to be inserted into ADMM scheme and achieves state-of-the-art results. Different from the task-specific ADMMs which usually require certain conditions of these specific modules and are prone to introduce unclear structures that may change the objective function, we rigorously prove the convergence of the established task-specific schemes according to objective function values and constraint violation with a flexible module and easy-to-calculate error control policy. Moreover, we analysis the convergence rate in terms of the iteration complexity. It is worth noting that our investigations not only reduce the gap between the practical implementations of ADMM and the strict convergence analyze for task-specific approaches but also provide a computationally feasible and theoretically guaranteed manner to address convex optimization in real-world application scenarios.

\section{The Proposed GO-ADMM}\label{sec:progpagation}

As aforementioned, most existing task-specific ADMM schemes perform the task-specific computational modules, instead of the non-differentiable subproblem (i.e., Eq.~\eqref{eq:y-iter}). Consequently, they fail to preserve the well-designed priors in the original optimization formulation. Notably, strict convergence properties (e.g., the global convergence to the original model and the exact convergence rate estimation) cannot be adequately guaranteed. In this section, we propose a new task-specific algorithmic framework, namely, the GO-ADMM, to address the above two fundamental issues.

Specifically, GO-ADMM aims to incorporate task information into the subproblem w.r.t. $\mathbf{x}$. The differentiable objective $f$ of the optimization model in Eq.~\eqref{eq:model} characterizes the fidelity/loss information of the task. In most learning and vision applications, the variables in this term are always measured after some given linear mappings (e.g., the degradation operation, mask matrix, transformation, and/or additional errors). Therefore, hereafter we further specify $f(\mathbf{x}):=l(\mathbf{Q}\mathbf{x})$, where $\mathbf{Q}\in \mathbb{R}^{p\times n}$ denotes a task-related linear mapping, and $l:\mathbb{R}^p\rightarrow\mathbb{R}$ refers to the measurement which is continuously differentiable and strongly convex. Many commonly used fidelity/loss functions (e.g., quadratic, exponential, and logistic losses) can be formulated as such specific $f$.

\subsection{Non-Euclidean Proximal Regularization} 

Instead of generating an iterative trajectory in Euclidean space as in standard ADMM, we first introduce a general proximal regularization for Eq.~\eqref{eq:x-itera-ne}. Specifically, we define a proximal term at the $k$-th iteration as $\mathcal{H}(\mathbf{x},\mathbf{x}^k,\mathbf{W})=\frac{1}{2}\|\mathbf{W}(\mathbf{x}-\mathbf{x}^k)\|_2^2$, where $\mathbf{W}$ defines a task-specific general metric (including Euclidean and non-Euclidean metrics) and will be specified for particular applications in Sec.~\ref{sec:application}. In this way, the $\mathbf{x}$-update reads as
\begin{equation}
\mathbf{x}^{k+1} = \arg\min\limits_{\mathbf{x}} \mathcal{L}_{\beta}(\mathbf{x},\mathbf{y}^k,\bm{\bm{\lambda }}^k)+\mathcal{H}(\mathbf{x},\mathbf{x}^k,\mathbf{W}).\label{eq:x-itera-ne}
\end{equation}
Particularly, $\mathbf{W}$ is a flexible scheme that can incorporate particular task information (e.g., mask, filter) or hyper-parameter into the optimization process. If $\mathbf{W}$ is a real number, $\frac{1}{2}\|\mathbf{W}(\x-\x^k)\|_2^2$ means Euclidean metric, while if we set $\mathbf{W}$ as a mask or filter, $\frac{1}{2}\|\mathbf{W}(\x-\x^k)\|_2^2$ implies non-Euclidean metric. In addition, the involvement of $\mathbf{W}$ also improves the computational performance in the differentiable subproblem (i.e., Eq.~\eqref{eq:x-itera-ne}), which will be clarified in the following subsection. 

\subsection{Differentiable Updating with Modules Ensemble}\label{subsec:differentiable_updating}

Now we design a new updating rule to further perform task-specific computations for the differentiable subproblem in Eq.~\eqref{eq:x-itera-ne}. Specifically, the $\mathbf{x}$-subproblem can be reformulated as the following linear system
\begin{equation}\label{eq:linear-system-x}
\begin{array}{l}
\nabla_{\mathbf{x}}\left(\mathcal{L}_{\beta}(\mathbf{x}^{k+1},\mathbf{y}^k,\bm{\bm{\lambda }}^k) + \mathcal{H}(\mathbf{x}^{k+1},\mathbf{x}^k,\mathbf{W})\right)\\
=\mathbf{Q}^\top\nabla l(\mathbf{Q}\mathbf{x}^{k+1})+(\mathbf{W}^\top\mathbf{W}+\beta\mathbf{A}^{\top}\mathbf{A})\mathbf{x}^{k+1}-\mathbf{s}^k=0,
\end{array}
\end{equation}
where $\mathbf{s}^k=\beta\mathbf{A}^{\top}(-\mathbf{B}\mathbf{y}^k+\mathbf{c}+\bm{\bm{\lambda }}^k/\beta)+\mathbf{W}^\top\mathbf{W}{\mathbf{x}}^k$. By defining the following operation on $\mathbf{x}$
\begin{equation*}
\begin{array}{l}
\mathcal{F}^k(\mathbf{x}):=(\mathbf{W}^\top\mathbf{W}+\beta\mathbf{A}^{\top}\mathbf{A})^{-1}\left(\mathbf{s}^k-\mathbf{Q}^{\top}\nabla l(\mathbf{Q}\mathbf{x}) \right),
\end{array}
\end{equation*}
and assuming that $(\mathbf{W}^\top\mathbf{W}+\beta\mathbf{A}^{\top}\mathbf{A})$ is invertible\footnote{We can easily obtain this invertible property by introducing a invertible metric matrix $\mathbf{W}^\top\mathbf{W}$.}, it is easy to see that solving Eq.~\eqref{eq:x-itera-ne} is equivalent to finding an approximate solution of the system $\mathbf{x}=\mathcal{F}^k(\x)$.

In standard ADMM scheme, one may directly adopt numerical techniques to solve this equation.\footnote{This solution can be derived with a closed-form solution or iterative methods. For quadratic cases, the approximate solution can be obtained either by an iterative method such as Preconditioned Conjugate Gradient (PCG)~\cite{eisenstat1981efficient} or direct method such as Cholesky factorization~\cite{chen2008algorithm}. For nonlinear cases, we can employ gradient descent or Newton-type methods to solve this system. } In contrast, we introduce an auxiliary variable $\hat{\mathbf{x}}^{k+1}$ to integrate the original updating trajectory and the (designed and/or trained) task-specific computations as follows
\begin{equation}\label{eq:inner-x-update1}
\hat{\mathbf{x}}^{k+1} = (1-\alpha)\mathbf{x}^{k}+\alpha\mathcal{D}^k(\mathbf{x}^{k}),
\end{equation}
where $\mathcal{D}^k$ denotes a given task-specific module and $\alpha\in[0,1]$ is an adaptive averaging parameter which reflects the influence of task-specific modules. Our algorithmic framework does not request any specific assumptions for the computational module $\mathcal{D}^k$. Indeed, we incorporate the guidance policy to automatically adjust the influence factor $\alpha$. That is, improper $\mathcal{D}^k$ can be eliminated. In a consequence, the formal updating of $\mathbf{x}^{k+1}$ reads as
\begin{equation}
\mathbf{x}^{k+1}=\mathcal{F}^k(\hat{\mathbf{x}}^{k+1}).\label{eq:update-x-formal}
\end{equation}

\subsection{The Guidance of Optimality (GO) Policy}\label{subsec:guidance}

To navigate the above task-specific iterations towards desired optimal solutions, we introduce a new policy, named Guidance of Optimality (GO), to identify the properly nested module and control the iteration trajectory. Specifically, we define an error term to measure the inexactness of our performed task-specific computation as follows
\begin{equation}\label{eq:ek}
\mathbf{e}^k(\mathbf{x}):=\nabla l(\mathbf{Q}\mathcal{F}^k(\mathbf{x}))-\nabla l(\mathbf{Qx}).
\end{equation} 
Upon together Eqs.~\eqref{eq:linear-system-x}, \eqref{eq:update-x-formal} and~\eqref{eq:ek}, we have
\begin{equation*}
\begin{aligned}
&\quad\nabla_{\mathbf{x}} {\mathcal{L}}_{\beta}( \mathbf{x}^{k+1},\mathbf{y}^k,\bm{\lambda }^k )+\mathbf{W}^\top\mathbf{W}(\mathbf{x}^{k+1}-\mathbf{x}^k)\\
& = \mathbf{Q}^\top \nabla l(\mathbf{Q}\mathcal{F}^k(\hat{\mathbf{x}}^{k+1} ))-\mathbf{Q}^\top \nabla l(\mathbf{Q}\hat{\mathbf{x}}^{k+1} )
=\mathbf{Q}^\top \mathbf{e}^k(\hat{\mathbf{x}}^{k+1}).
\end{aligned}
\end{equation*}
Therefore, $\|\mathbf{Q}^\top  \mathbf{e}_k (\hat{\mathbf{x}}^{k+1})\|_2$ can be quantitatively regarded as the difference between the task-specific solution of $\mathbf{x}$-subproblem generated by GO-ADMM and the exact solution obtained by standard ADMM in Eq.~\eqref{eq:x-iter}, in the sense of the partial gradient residual from the augmented Lagrangian function. Inspired by this observation, we introduce the following relaxed condition as our guidance policy:
\begin{equation}\label{eq:ek-condition}
\|{\mathbf{e}}^k(\hat{\mathbf{x}}^{k+1})\|_2\leq \eta\|{\mathbf{e}}^k(\hat{\mathbf{x}}^{k})\|_2,
\end{equation}
where $\eta\in(0,1)$ is a constant. In this way, if the task-specific computation in Eq.~\eqref{eq:update-x-formal} satisfies the inequality in Eq.~\eqref{eq:ek-condition}, we actually obtain a proper updating for the differentiable subproblem. Rather, we should reduce the influence of the nested module, until it does not work for our iterations. So in the worst case, GO-ADMM will temporarily reduce to a standard ADMM at some iterations if improper computational modules are utilized. Overall, our GO-ADMM algorithmic framework is summarized in Alg.~\ref{alg:GO-ADMM}.

\begin{algorithm}[htb]
	\caption{GO-ADMM for Eq.~\eqref{eq:model}}
	\label{alg:GO-ADMM}
	\begin{algorithmic}[1]
		\REQUIRE Input $\mathbf{x}^0$, $\mathbf{y}^0$, $\bm{\bm{\lambda }}^0$,  $\{\mathcal{D}^k\}$, and necessary parameters.
		\ENSURE $\mathbf{x}^{\ast}$, $\mathbf{y}^{\ast}$, and $\bm{\bm{\lambda }}^{\ast}$.
		\WHILE{not converged}
		\STATE $\hat{\mathbf{x}}^{k+1} = (1-\alpha){\mathbf{x}}^{k}+\alpha\mathcal{D}^k(\mathbf{x}^{k})$.\label{step:network}
		\IF{$ \|\mathbf{e}^k(\hat{\mathbf{x}}^{k+1})\|_2 > \eta\|\mathbf{e}^k({\hat{\mathbf{x}}}^{k}) \|_2$}\label{step:check}
		\STATE Set $\alpha=\rho\alpha$ with $0 < \rho \ll 1$ and go to Step~\ref{step:network}.
		\IF{$\alpha<\epsilon$}
		\STATE $\hat{\mathbf{x}}^{k+1}=\mathtt{Num}(\mathbf{x}^k;\mathbf{s}^k,\beta)$
		\ENDIF
		\ENDIF
		\STATE $\mathbf{x}^{k+1}=\mathcal{F}^k(\hat{\mathbf{x}}^{k+1})$.
		\STATE 
		$\mathbf{y}^{k+1}=\arg\min\limits_{\mathbf{y}}\mathcal{L}_{\beta}(\mathbf{x}^{k+1},\mathbf{y},\bm{\bm{\lambda }}^{k})$.
		\STATE
		$\bm{\bm{\lambda }}^{k+1}=\bm{\bm{\lambda }}^k-\beta\left(\mathbf{A}\mathbf{x}^{k+1}+\mathbf{B}\mathbf{y}^{k+1}-\mathbf{c}\right)$.
		\ENDWHILE
	\end{algorithmic}
\end{algorithm}

\begin{remark}
	We argue that our error condition (i.e., Eq.~\eqref{eq:ek-condition}) can always be satisfied in Alg.~\ref{alg:GO-ADMM}, for parameter $\alpha$ is dramatically decreasing while iterations. Thus, in the worst case, Eq.~\eqref{eq:ek-condition} will be satisfied when $\alpha\to 0$. Actually, equation $ \hat{\mathbf{x}}^{k+1} = (1-\alpha)\mathbf{x}^{k}+\alpha\mathcal{D}^k(\mathbf{x}^{k}) $ reduces to $\mathbf{x}^k$ as $\alpha\to 0$. Then, iteration of $\mathbf{x}$-subproblem is a pure numerical scheme and one may directly adopt numerical techniques to solve this equation. 
	We summarize the numerical algebra solver as $\hat{\mathbf{x}}^{k+1}=\mathtt{Num}(\x^k;\mathbf{s}^k,\beta)$ and explain this numerical scheme as follows. 
   For quadratic cases of function $f$, the approximate system $\x=\F^k(\x)$ is linearized and its solution can be obtained either by an iterative method such as Preconditioned Conjugate Gradient (PCG) or direct method such as Cholesky factorization. For example, we consider the instance $\mathbf{A} = 1$, $\mathbf{B} = -1$, $\mathbf{c} = 0$, $\mathbf{W} = 0$, $l(\x)=1/2 \x^2$. For this case, we have $\mathcal{F}^k(\x)=(\mathbf{s}^k-\x)/\beta$. Thus, $\x$-subproblem becomes $\x = (\mathbf{s}^k-\x)/\beta $ and the solution can be derived by closed-form solution, i.e., $\hat{\mathbf{x}}^{k+1} = \x = \mathbf{s}^k/(1+\beta)$. With the defined error term $\mathbf{e}^k(\mathbf{x}):=\nabla l(\mathbf{Q}\mathcal{F}^k(\mathbf{x}))-\nabla l(\mathbf{Qx})$, the above closed-form solution implies 
   \begin{equation*}\label{eq:ineq_ek}
   \begin{aligned}
   \|\mathbf{e}^k(\hat{\x}^{k+1})\| &= \|\nabla l(\mathbf{Q}\mathcal{F}^k(\hat{\x}^{k+1}))-\nabla l(\mathbf{Q}\hat{\mathbf{x}}^{k+1})\|\\
   & \leq \|L\mathbf{Q}(\mathcal{F}^k(\hat{\x}^{k+1})-\hat{\mathbf{x}}^{k+1})\|\\
   & = \left\|L\mathbf{Q}\left(\frac{\mathbf{s}^k}{1+\beta}-\frac{\mathbf{s}^k-\hat{\mathbf{x}}^{k+1}}{\beta}\right)\right\|\\
   &=0.
   \end{aligned}
   \end{equation*}
   Based on the above formulation, we have $\|\mathbf{e}^k(\hat{\x}^{k+1})\| \leq \eta \|\mathbf{e}^k(\hat{\x}^{k})\|$. Also, we can apply PCG or Cholesky factorization to obtain the approximate solution $\hat{\x}^{k+1}$ until satisfying $\|\mathbf{e}^k(\hat{\x}^{k+1})\| \leq \eta \|\mathbf{e}^k(\hat{\x}^{k})\|$. For nonlinear cases, we can employ gradient descent or Newton-type methods to find $\hat{\mathbf{x}}^{k+1}$ satisfying Eq.~\eqref{eq:ek-condition}. 
\end{remark}
\begin{remark}
	As for $\mathcal{D}^k$, our algorithmic framework actually does not rely directly on any specific assumptions for these computational modules and the guidance policy in Section~\ref{sec:progpagation} can automatically reduce the influence up to reject these improper $\mathcal{D}^k$. 
\end{remark}

\section{Theoretical Investigations}

In this section, we investigate the convergence behaviors of task-specific iterations within the proposed GO-ADMM paradigm. Rather than enforcing restrictive constraints on these nested modules, here we consider the following mild assumptions on the function $l$ defined in Sec.~\ref{sec:progpagation}.
\begin{assum}
	For $\forall \mathbf{z}^1,\mathbf{z}^2\in\mathbb{R}^p$, $l(\cdot)$ satisfies that
	\begin{equation}\label{eq:StrongConvex}
	\theta\|\mathbf{z}^1-\mathbf{z}^2\|_2^2\leq (\mathbf{z}^1-\mathbf{z}^2)^{\top}(\nabla l(\mathbf{z}^1)-\nabla l(\mathbf{z}^2)), 
	\end{equation}
	and its gradient is Lipschitz continuous such that 
	\begin{equation}\label{eq:fLipschitz}
	\|\nabla l(\mathbf{z}^1)-\nabla l(\mathbf{z}^2)\|_2\leq L\|\mathbf{z}^1-\mathbf{z}^2\|_2,
	\end{equation}
	where $\theta$ and $L$ are two positive constants.
\end{assum}

The following definitions are used in the sequel. Let $\mathbf{\Omega}:={\mathbb{R}}^{n}\times {\mathbb{R}}^{m}\times {\mathbb{R}}^{\ell}$ and  denote $\bm{w}=( \mathbf{x}, \mathbf{y}, \bm{\bm{\lambda } })^{\top}\in \mathbf{\Omega}$. Then define the operator $\mathbf{F}(\bm{w}):{\mathbb{R}}^{(n+m + \ell )}\rightarrow{\mathbb{R}}^{(n+ m + \ell )}$ and the matrix ${\mathbf{M}}\in {\mathbb{R}}^{(n+ m + \ell )\times (n+ m + \ell )}$ as follows:
\begin{equation}\label{eq:notation}
\begin{array}{l}
\!\!\mathbf{F}(\bm{w})\!=\! \left( \mathbf{Q}^\top \nabla l( \mathbf{Q} \mathbf{x}) \!-\! \mathbf{A}^\top  \bm{\bm{\lambda } },\ -\mathbf{B}^\top  \bm{\bm{\lambda } },\  \mathbf{A} \mathbf{x} \!+\! \mathbf{B} \mathbf{y} - \mathbf{c} \right)^{\top},\\
\!\!\mathbf{M} = \mathtt{diag}\left(\mathbf{W}^\top\mathbf{W},\ \beta \mathbf{B}^\top  \mathbf{B},\ \frac{1}{\beta}\mathbf{I}_{\ell\times \ell}\right),
\end{array}
\end{equation}
where $\mathtt{diag}(\cdot)$ denotes the diagonal array. Note that ${\mathbf{M}}$ is not necessarily positive definite because the matrix $\mathbf{B}$ in Eq.~\eqref{eq:model} is not assumed to be full column rank. 

We first reformulate Eq.~\eqref{eq:model} into a Variational Inequality (VI). This reformulation helps to analyze convergent properties via VI approach. As initiated in~\cite{he20121}, the Lagrangian function of Eq.~(1) is defined as
\begin{equation*}
L(\x,\y,\bm{\lambda}) = f(\mathbf{x}) + g(\mathbf{y}) - \bm{\lambda}^{\top}(\mathbf{A}\mathbf{x}+\mathbf{B}\mathbf{y} -\mathbf{c}),
\end{equation*}
where $\bm{\lambda}$ is Lagrangian multiplier. Let $\bm{\omega}^{\ast}=(\mathbf{x}^{\ast},\mathbf{y}^{\ast},\bm{\lambda}^{\ast})^{\top}$ be a saddle point of the Lagrangian function. Then for any $\bm{\omega}=(\mathbf{x},\mathbf{y},\bm{\lambda})\in\Omega$, $\bm{\omega}^{\ast}\in\Omega$ satisfies
\begin{equation}\label{eq:optimal_condition}
\left\{ 
\begin{aligned}
f(\mathbf{x}) - f(\mathbf{x}^{\ast}) + (\mathbf{x}-\mathbf{x}^{\ast})^{\top}(-\mathbf{A}^{\top}\bm{\lambda}^{\ast})&\geq 0,\\
g(\mathbf{y}) - g(\mathbf{y}^{\ast}) + (\mathbf{y}-\mathbf{y}^{\ast})^{\top}(-\mathbf{B}^{\top}\bm{\lambda}^{\ast})&\geq 0,\\
(\bm{\lambda}-\bm{\lambda}^{\ast})^{\top}(\mathbf{A}\mathbf{x}^{\ast}+\mathbf{B}\mathbf{y}^{\ast}-\mathbf{c})&\geq 0.
\end{aligned}
\right.
\end{equation}
With specified $f(\mathbf{x}):=l(\mathbf{Q}\mathbf{x})$, then $f(\mathbf{x}) - f(\mathbf{x}^{\ast}) = (\mathbf{x}-\mathbf{x}^{\ast})^{\top}\mathbf{Q}^{\top}\nabla l(\mathbf{Q}\mathbf{x}^{\ast})$. Note that these first-order optimal conditions (i.e. Eq.~\eqref{eq:optimal_condition}) can be written in a compact form
\begin{equation}\label{eq:Algo_General_VI}
{\hbox{VI}}(\mathbf{\Omega}, \mathbf{F}): g(\mathbf{y})-g(\mathbf{y}^{\ast})+ ( \bm{w}-\bm{w}^{\ast} )^\top  \mathbf{F}(\bm{w}^{\ast})\ge 0. 
\end{equation}
Thus, the VI can be designed as finding $\bm{w}^{\ast}=( \mathbf{x}^{\ast},\mathbf{y}^{\ast},\bm{\lambda }^{\ast})^{\top} \in \mathbf{\Omega}$, $\forall\bm{w}\in\mathbf{\Omega}$, that satisfies the above VI inequality~\eqref{eq:Algo_General_VI}. Actually, throughout the work, we consider the case of practical interest that the KKT solution set of Eq.~\eqref{eq:Algo_General_VI} is nonempty. Now we are ready to present our theoretical results. Specifically, to prove the convergence of the sequence generated by GO-ADMM, it is crucial to analyze how the residual $\|\mathbf{e}^k (\mathbf{x})\|_2$ evolves according to the iterations. For this reason, we first provide the following proposition. 
\begin{prop}\label{prop:ek_relation}
	Let $\{\mathbf{e}^k(\mathbf{\mathbf{x}})\}$ be the sequence defined in Eq.~\eqref{eq:ek},  $\{\hat{\mathbf{\mathbf{x}}}^{k}\}$ and $\{\bm{w}^k\}$ are the sequences generated by GO-ADMM. If constant $\eta$ satisfies %the following inequality
	\begin{equation}\label{eq:etaCriterion}
	0 <\eta < \sqrt{2\theta}/\left(\sqrt{2\theta}+L\|\mathcal{N}\|_2\right)\in(0,1),
	\end{equation}
	with $\mathcal{N}\!=\!\mathbf{Q}(\mathbf{W}^{\top}\mathbf{W}\!+\!\beta\mathbf{A}^{\top}\mathbf{A})^{-1}\left[\mathbf{W}^{\top},\sqrt{\beta}\mathbf{A}^{\top}\right]\in\mathbb{R}^{p\times(n+l)}$,  then we have the following inequality
	\begin{equation}\label{eq:cretirion_relation}
	\| \mathbf{e}^k (\hat{\mathbf{x}}^{k+1})  \|_2 \le \eta \| \mathbf{e}^{k-1} (\hat{\mathbf{x}}^{k})  \|_2 + \eta\gamma \|\bm{w}^{k-1}-\bm{w}^k\|_{\mathbf{M}},
	\end{equation}
	where $ \gamma = L \| {\mathcal{N}}\|_2$. 
\end{prop}

%%%%%%%%%%%%%%%%%%%%%

Now we prove the convergence of the sequence generated by GO-ADMM. To simplify the notation, we denote 
\begin{equation}\label{eq:wbar}
\!\bar{\bm{w}}^{k+1}=\left(\x^{k+1},\y^{k+1},\bm{\lambda}^k-\beta(\mathbf{A}\x^{k+1}+\mathbf{B}\y^k-\c)\right)^{\top}.
\end{equation}
In fact, this notation is not required to be practically calculated when implementing GO-ADMM. Then in the following proposition, we analyze the difference between $\bar{\bm{w}}^k$ defined in Eq.~\eqref{eq:wbar} and the solution point formulated by Eq.~\eqref{eq:Algo_General_VI}.
\begin{prop}\label{prop:re_Var_Inequality}
	Let $\left\{ \bm{w}^k \right\}$ be the sequence generated by the GO-ADMM, $\bar{\bm{w}}^k$ and ${\mathbf{M}}$ are defined in Eqs.~\eqref{eq:wbar} and~\eqref{eq:notation} respectively. Then, for all $\bm{w}\in\mathbf{\Omega}$, it holds that
	\begin{equation}\label{eq:re_Var_Inequality}
	\begin{array}{l}
	g(\y^{k+1}) - g(\y) + (\bar{\bm{w}}^{k+1}-\bm{w})^{\top}\mathbf{F}(\bar{\bm{w}}^{k+1})\\
	\leq \bm{q}_k(\x)^\top \mathbf{e}^k(\hat{\mathbf{x}}^{k+1})-\Delta_{\mathbf{M}}(\bm{w},\bm{w}^k,\bm{w}^{k+1}),
	\end{array}
	\end{equation}
	where $\Delta_{\mathbf{M}}(\bm{w},\bm{w}^k,\bm{w}^{k+1}) = \frac{1}{2}(\|\bm{w}-\bm{w}^{k+1}\|_{\mathbf{M}}^2-\|\bm{w}-\bm{w}^k\|_{\mathbf{M}}^2+\|\bm{w}^k-\bm{w}^{k+1}\|_{\mathbf{M}}^2)$ and  $\bm{q}_k(\x)=\mathbf{Q} ( \mathbf{x}^{k+1} - \mathbf{x} )$.
\end{prop}

The difference between the inequality in Eq.~\eqref{eq:re_Var_Inequality} and the variational inequality reformulation in Eq.~\eqref{eq:Algo_General_VI} reflects the difference between the point ${\bar{\bm{w}}}^k$ and a solution point $\bm{w}^{\ast}$. For the right-hand side of Eq.~\eqref{eq:re_Var_Inequality}, the second term (i.e., three quadratic terms) is easy to be manipulated over different indicators by algebraic operations, but it is not clear about how to control the crossing term (i.e., $\bm{q}_k(\mathbf{x})^{\top}\mathbf{e}^k(\hat{\mathbf{x}}^{k+1})$) towards the eventual goal, i.e., illustrate the convergence of the sequence $\{\bm{w}^k\}$. We thus explore this term particularly and derive that the sum of these crossing terms over $K$ iterations can be bounded by some quadratic terms as well. This result is summarized in the following proposition.

\begin{prop}\label{prop:sum_Corssing_Term}
	Let $\left\{ \bm{w}^k \right\}$ be the sequence generated by GO-ADMM. For all $\mathbf{x}  \in \mathbb{R}^n$, $K > 1$ and $\mu>0$, it holds that
	%	\footnotesize{
	\begin{equation}\label{eq:sum_Crossing_Term}
	\begin{array}{l}
	\sum\limits_{k = 1}^{K} \bm{q}_k(\mathbf{x})^\top\mathbf{e}^k(\hat{\mathbf{x}}^{k+1})\le\frac{\eta}{1-\eta} \left\{ \frac{1}{2\mu} \sum\limits^{K-1}_{k = 1} \gamma^2 \|\bm{w}^k-\bm{w}^{k+1}\|_{\mathbf{M}}^2 \right.\\
	\left. +\frac{\mu}{2} \sum\limits^{K}_{k = 1} \| \bm{q}_k(\mathbf{x})\|_2^2+\frac{1}{2\mu} \left(\| \mathbf{e}^0(\hat{\mathbf{x}}^1) \|_2 + \gamma \|\bm{w}^0-\bm{w}^{1}\|_{\mathbf{M}} \right)^2 \right\}.
	\end{array}
	\end{equation}
\end{prop}

Now we establish the convergence results of GO-ADMM in the following theorem.
\begin{thm}\label{thm:Convergence}
	Let $\left\{ \bm{w}^k \right\}$ be the sequence generated by the GO-ADMM and denote $\bm{\Omega}^{\ast}$ as the solution set of the variational inequality in Eq.~\eqref{eq:Algo_General_VI}. Then, we have the following assertions:
	\begin{enumerate}
		\item $\|\mathbf{e}^{k} (\hat{\mathbf{x}}^{k+1}) \|_2\overset{k\rightarrow \infty}{\longrightarrow} 0$, and $\| \mathbf{B}( \mathbf{y}^{k} - \mathbf{y}^{k+1} )\|_2 \overset{k\rightarrow \infty}{\longrightarrow} 0$.
		\item $\|\mathbf{A} \mathbf{x}^{k+1} + \mathbf{B} \mathbf{y}^{k+1} - \mathbf{c}\|_2 \overset{k\rightarrow \infty}{\longrightarrow} 0$, and $f(\mathbf{x}^{k+1}) + g(\mathbf{y}^{k+1}) \overset{k\rightarrow \infty}{\longrightarrow} f(\mathbf{x}^{\ast}) + g (\mathbf{y}^{\ast})$ for any given $\bm{w}^{\ast}\in\mathbf{\Omega}^{\ast}$.
	\end{enumerate}
\end{thm}
It is easy to verify that $\bm{w}^{k+1}$ is a solution of Eq.~\eqref{eq:Algo_General_VI} if and only if $\|\bm{w}^k-\bm{w}^{k+1}\|_{\mathbf{M}}^2=0$ and $\|\mathbf{e}^k(\hat{\mathbf{x}}^{k+1})\|_2^2=0$. Hence, it is reasonable to measure the accuracy of the iterate $\bm{w}^{k+1}$ by $\|\bm{w}^k-\bm{w}^{k+1}\|_{\mathbf{M}}^2$ and $\|\mathbf{e}^k(\hat{\mathbf{x}}^{k+1})\|_2^2$.
\begin{cor}\label{corollary}
	The upper bounds of $\min_{1\le k\le K}\{\|\bm{w}^k-\bm{w}^{k+1}\|_{\mathbf{M}}^2\}$ and $\min_{1\le k\le K}\{\|\mathbf{e}^k(\hat{\mathbf{x}}^{k+1})\|_2^2\}$ is in order of  ${\cal{O}}(\frac{1}{K})$. That is, our proposed GO-ADMM actually obtains ${\cal{O}}(\frac{1}{K})$ worst-case convergence rate in a non-ergodic sense.
\end{cor}
\begin{remark}
	 Instant of making assumptions on $\mathbf{A}$ directly, we suppose that $(\mathbf{W}^\top\mathbf{W}+\beta\mathbf{A}^{\top}\mathbf{A})$ is invertible. This invertible property can be easily obtained by introducing an invertible metric matrix $\mathbf{W}^\top\mathbf{W}$. Second, we must clarify that in this work we provide general theoretical results. The convergence regarding $\mathbf{y}$ is actually $\| \mathbf{B}( \mathbf{y}^{k} - \mathbf{y}^{k+1} )\|_2 \overset{k\rightarrow \infty}{\longrightarrow} 0$, without making any assumption on $\mathbf{B}$. Trivially, if $\mathbf{B}$ is full column rank, we can obtain $\|\mathbf{y}^{k} - \mathbf{y}^{k+1}\|_2\overset{k\rightarrow \infty}{\longrightarrow} 0$ immediately. It is worth mentioning that, in image processing tasks (the object is usually formulated as $f(\mathbf{x}) + g(\mathbf{Px})$), we often reformulate these problems to general forms (i.e, $f(\mathbf{x})+g(\mathbf{y})$, s.t., $\mathbf{Px}-\mathbf{y}=\mathbf{0}$) by introducing $\mathbf{Px}=\mathbf{y}$ where $\mathbf{y}$ is an auxiliary variable. In this case, $\mathbf{B}$ is denoted as a negative identity matrix (i.e., $\mathbf{B}=-\mathbf{I}$) which must be full column rank.
\end{remark}

\begin{figure*}[t]
	\centering
	\begin{tabular}{c@{\extracolsep{0.2em}}c@{\extracolsep{0.2em}}c@{\extracolsep{0.2em}}c@{\extracolsep{0.2em}}c}
		\includegraphics[height=0.157\textwidth]{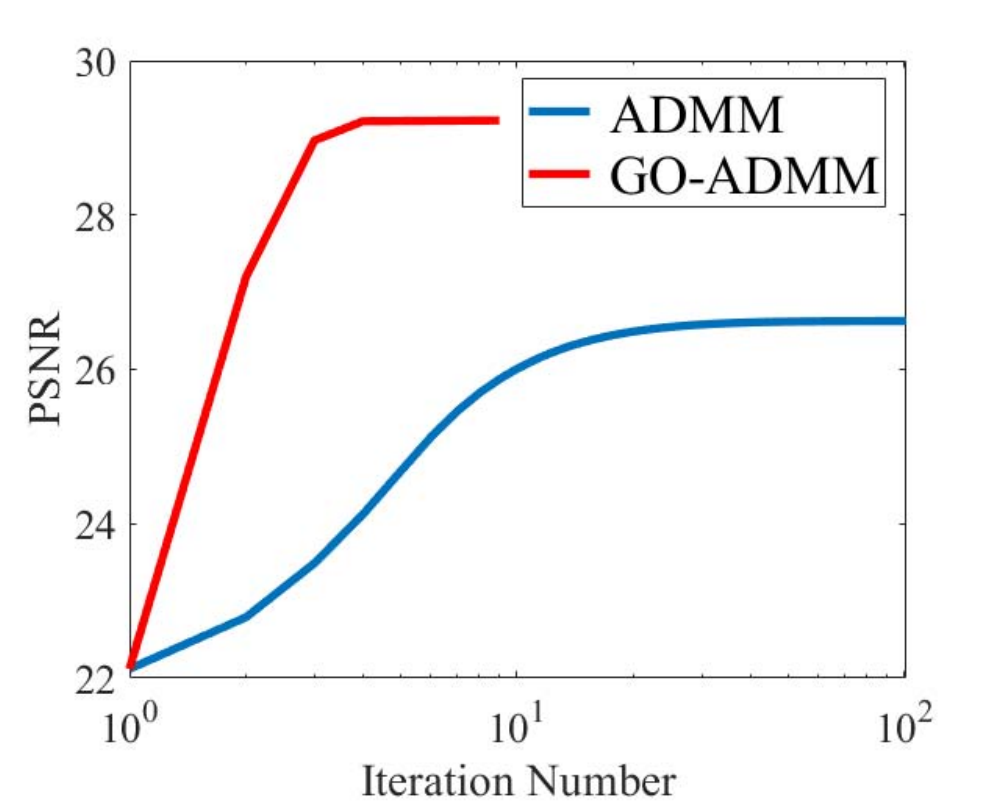}&
		\includegraphics[height=0.157\textwidth]{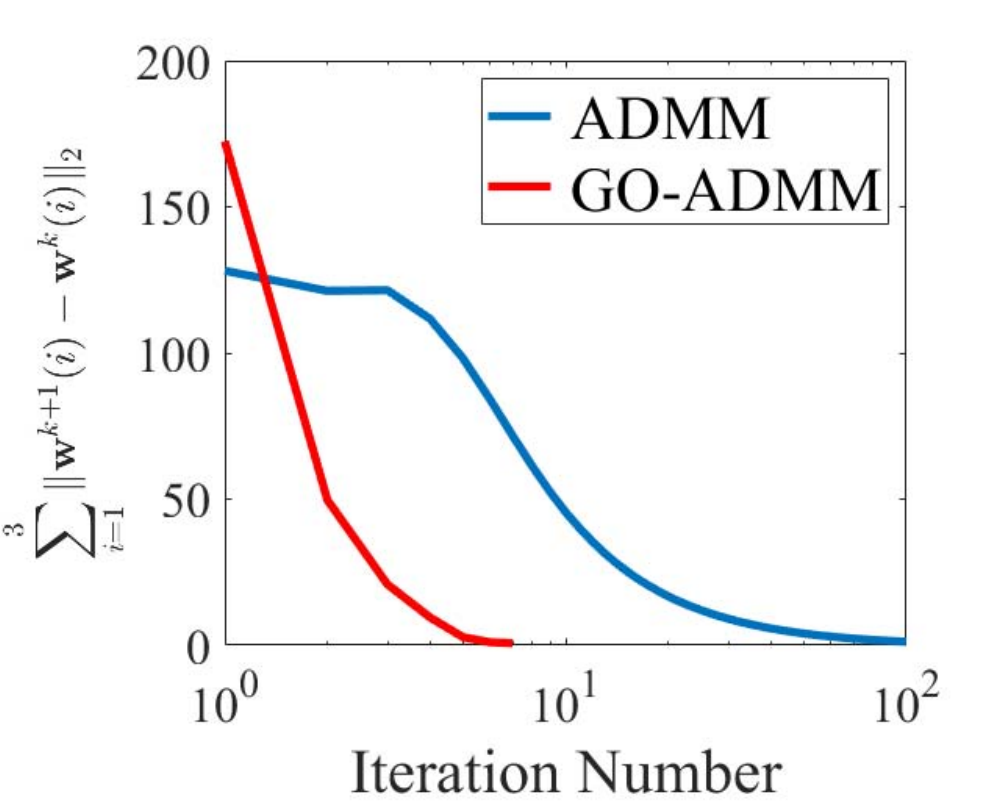}&
		\includegraphics[height=0.157\textwidth]{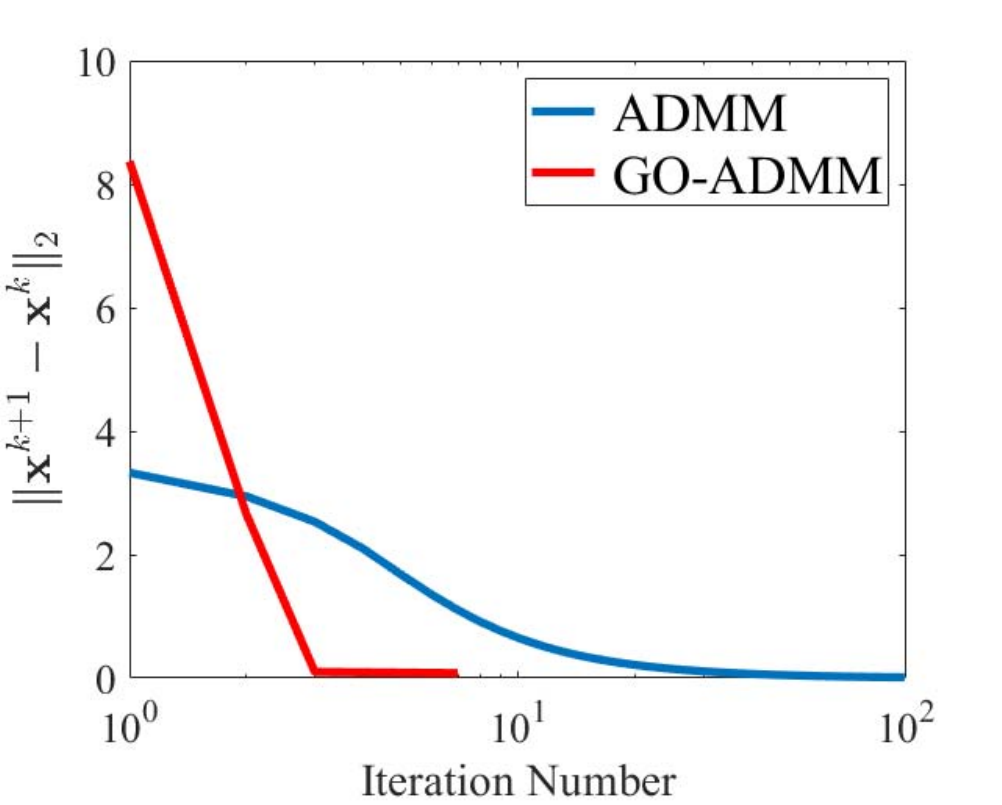}&
		\includegraphics[height=0.157\textwidth]{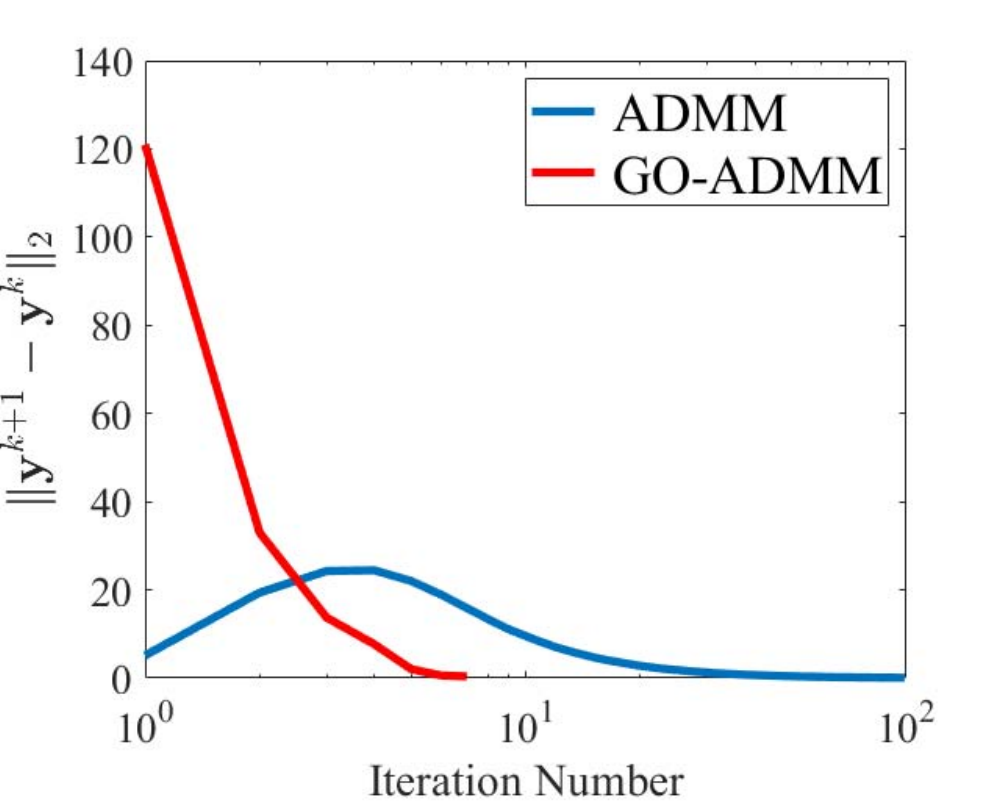}&
		\includegraphics[height=0.157\textwidth]{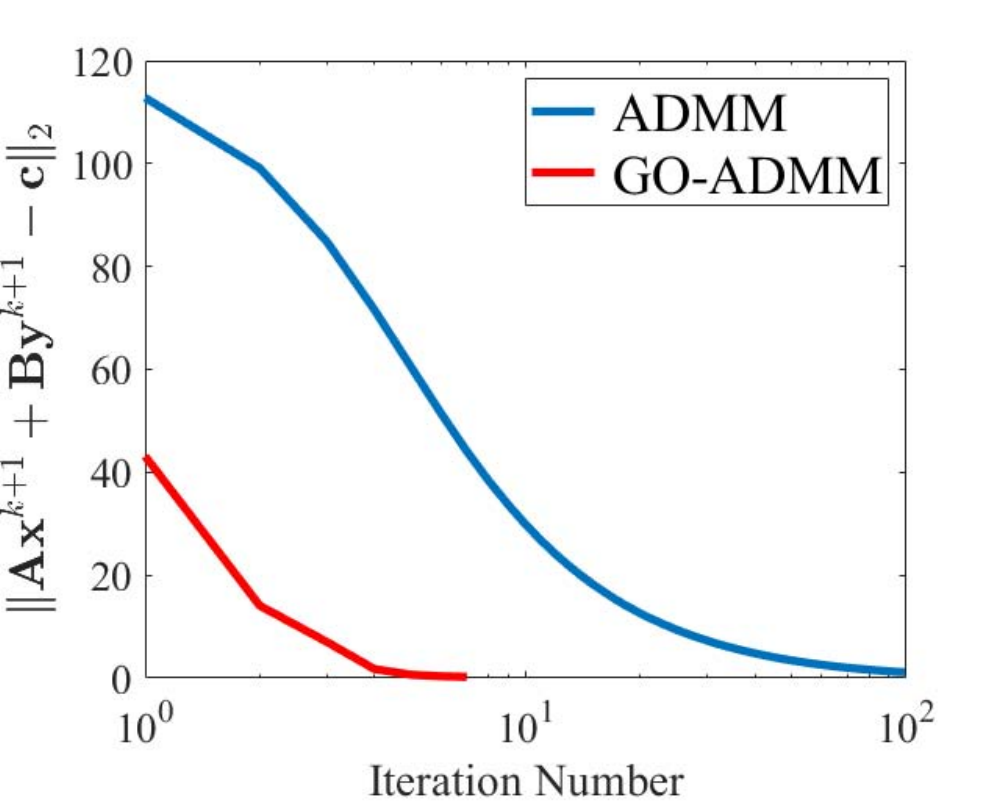}\\
		\footnotesize (a) &\footnotesize (b) & \footnotesize (c) &\footnotesize (d)&\footnotesize (e)
	\end{tabular}
	\caption{Comparing iteration behaviors of GO-ADMM with standard ADMM for image denoising task in $20\%$ noise level. The PSNR, iterative error ($\sum_{i=1}^{3}\|\bm{w}^{k+1}(i)-\bm{w}^k(i)\|_2=\|\mathbf{x}^{k+1}-\mathbf{x}^k\|_2+\|\mathbf{y}^{k+1}-\mathbf{y}^k\|_2+\|\bm{\lambda}^{k+1}-\bm{\lambda}^k\|_2$), $\|\mathbf{x}^{k+1}-\mathbf{x}^k\|_2$, $\|\mathbf{y}^{k+1}-\mathbf{y}^k\|_2$ and $\|\mathbf{Ax}^{k+1}+\mathbf{By}^{k+1}-\mathbf{c}\|_2$ are plotted in subfigures (a)-(e), respectively. } 
	\label{fig:CompNum}
\end{figure*}

\begin{figure*}[t]
	\centering
	\begin{tabular}{c@{\extracolsep{0.3em}}c@{\extracolsep{0.1em}}c@{\extracolsep{0.3em}}c@{\extracolsep{0.1em}}c}
		\includegraphics[height=0.158\textwidth]{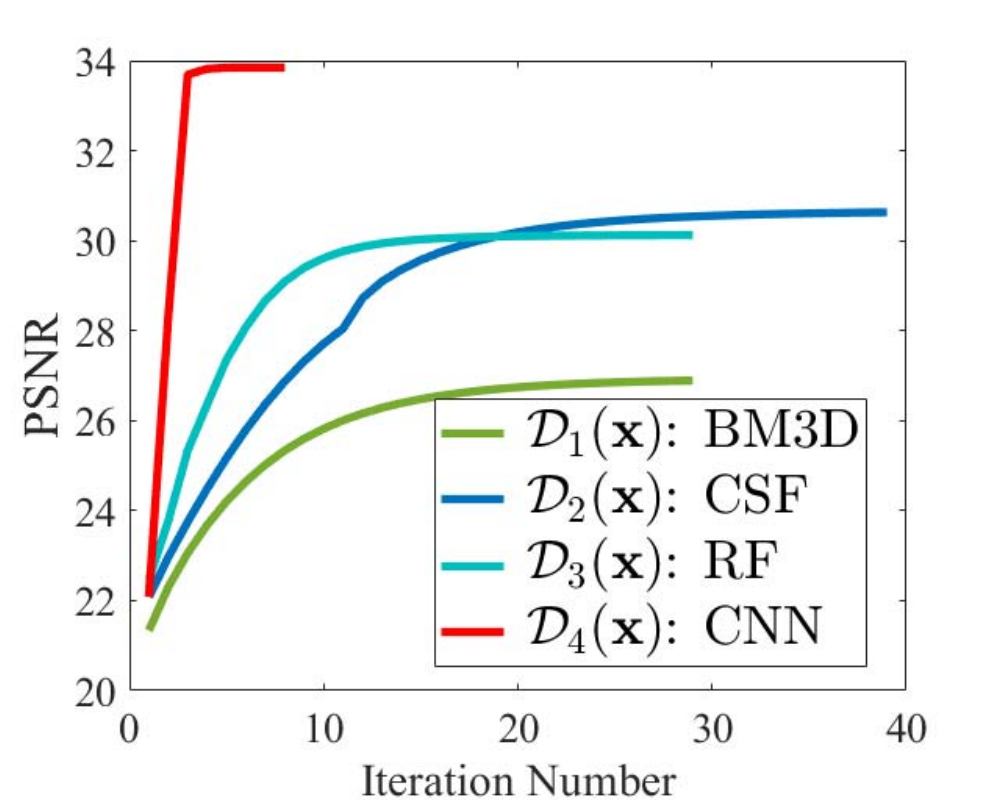}&
		\includegraphics[height=0.158\textwidth]{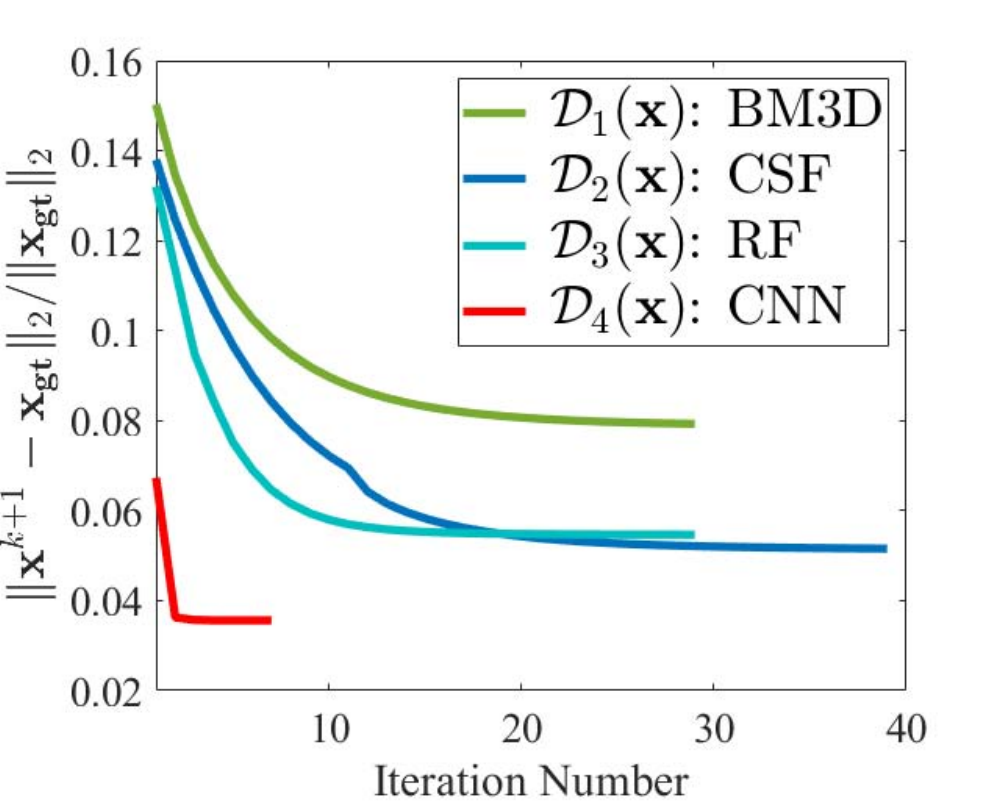}&
		\includegraphics[height=0.158\textwidth]{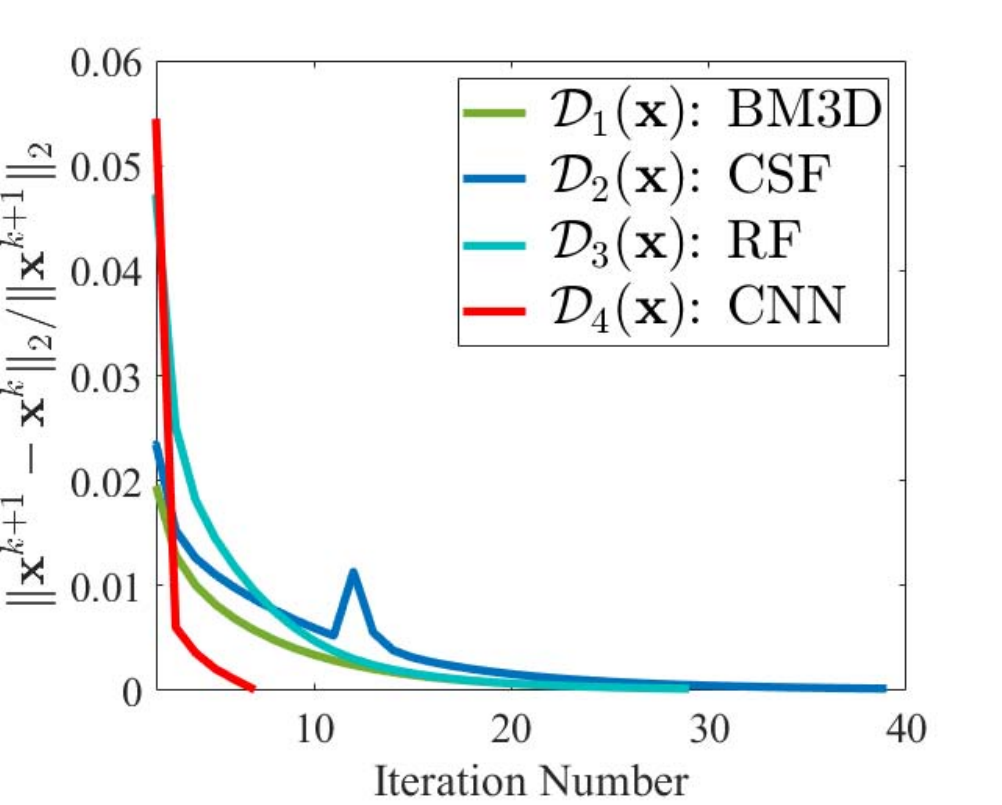}&
		\includegraphics[height=0.158\textwidth]{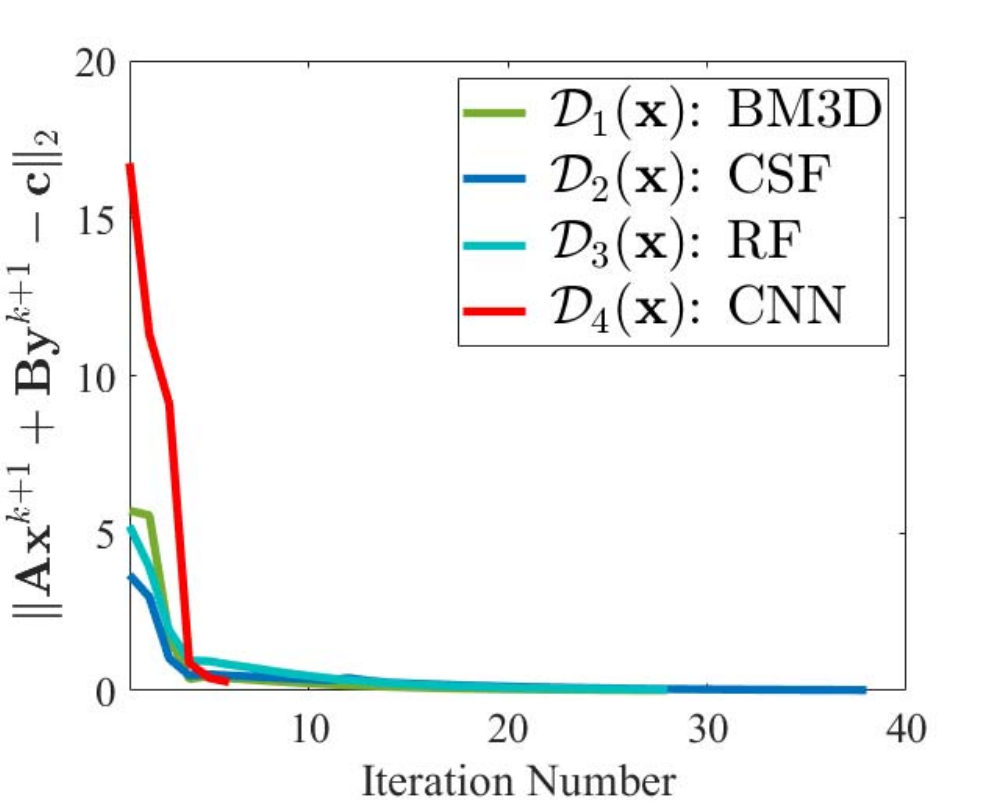}&
		\includegraphics[height=0.158\textwidth]{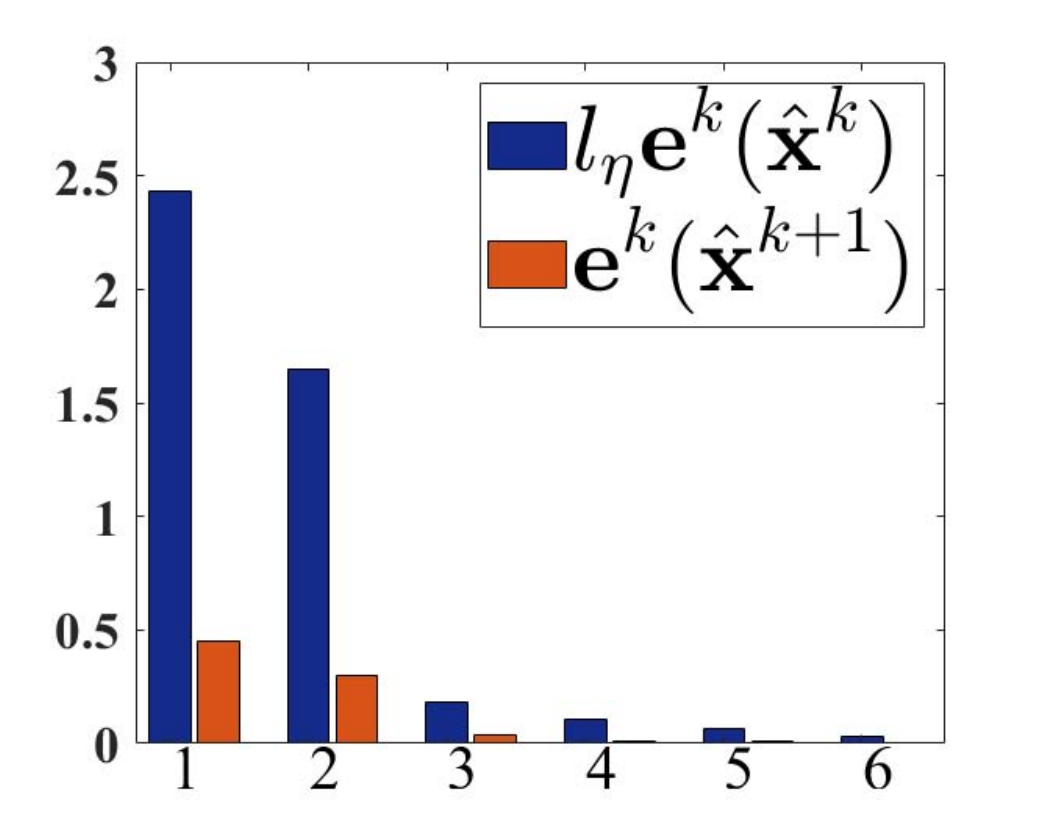}\\
	\end{tabular}
	\caption{Illustrating the convergence behaviors of GO-ADMM with different $\mathcal{D}$ in image denoising task. Denote $\mathcal{D}_{\mathtt{BM3D}}^k$, $\mathcal{D}_{\mathtt{CSF}}^k$, $\mathcal{D}_{\mathtt{RF}}^k$ and $\mathcal{D}_{\mathtt{CNN}}^k$ as the conditions with BM3D, CSF, RF and CNN respectively. The performance metric (i.e., PSNR), relative error ($\|\mathbf{x}^{k+1}-\mathbf{x}^{k}\|_2/\|\mathbf{x}^{k+1}\|_2$), reconstruction error ($\|\mathbf{x}^{k+1}-\mathbf{x}_{gt}\|_2/\|\mathbf{x}_{gt}\|_2$), iterative errors ($\sum_{i=1}^{3}\|\bm{w}^{k+1}(i)-\bm{w}^k(i)\|_2=\|\mathbf{x}^{k+1}-\mathbf{x}^k\|_2+\|\mathbf{y}^{k+1}-\mathbf{y}^k\|_2+\|\bm{\lambda}^{k+1}-\bm{\lambda}^k\|_2$) of the four strategies and the error control condition under $\mathcal{D}_{\mathtt{CNN}}^k$ are plotted in the subfigures (a)-(e), respectively.}
	\label{fig:CompNets}
\end{figure*}

\section{Applications}\label{sec:application}

As a nontrivial byproduct, we first demonstrate how to apply GO-ADMM to image restoration problems in real-world low-level vision applications, such as image denoising, inpainting, and compressed sensing MRI tasks. Then, we illustrate the implementation of GO-ADMM to tackle more difficult rain streaks removal problems that we estimate rain streaks and background simultaneously.

\textbf{Image Restoration}: Individually, we consider the following Total Variation (TV) minimization model that is popularly performed in various low-level computer vision tasks, especially in image processing problems. Then we define the model as
\begin{equation*}\label{eq:model_denoise}
\begin{array}{l}
\min\limits_{\mathbf{x}} \frac{1}{2}\|\mathbf{Q}\mathbf{x}-\mathbf{b}\|_2^2 + \mu\|\nabla\mathbf{x}\|_1,
\end{array}
\end{equation*}
where $\mathbf{x}$, $\mathbf{b}$ denote the latent and observed image respectively, and $\mathbf{Q}$ is a linear operator. $\nabla =[\nabla_h;\nabla_v]$ represents the gradient of $\mathbf{x}$ in horizontal and vertical directions. By introducing variable $\mathbf{u}=[\mathbf{u}_h;\mathbf{u}_v]$, the above equation can be transformed into the following form
\begin{equation}\label{eq:model_admm}
\begin{array}{l}
\min\limits_{\mathbf{x},\mathbf{u}} \frac{1}{2}\|\mathbf{Q}\mathbf{x}-\mathbf{b}\|_2^2 + \mu\|\mathbf{u}\|_1,\
s.t., \nabla\mathbf{x}-\mathbf{u} = \mathbf{0}.
\end{array}
\end{equation}
Next, applying GO-ADMM to Eq.~\eqref{eq:model_admm}, we have that
\begin{equation*}
\begin{array}{l}
\!\!\mathcal{F}^k(\mathbf{x})\!=\!\left(\mathbf{W}^\top\mathbf{W}\!+\!\beta(\nabla_h^{\top}\nabla_h\!+\!\nabla_v^{\top}\nabla_v)\right)^{\!-1}\left(\mathbf{s}^k\!-\!\mathbf{Q}^{\top}(\mathbf{Q}\mathbf{x}\!-\!\mathbf{b})\right),\ 
\end{array}
\end{equation*}
where $\mathbf{s}^k = \beta(\nabla_h^{\top}\mathbf{u}_h+\nabla_v^{\top}\mathbf{u}_v^k)+(\nabla_h^{\top}\bm{\bm{\lambda }}_h^k+\nabla_v^{\top}\bm{\bm{\lambda }}_v^k) + \mathbf{W}^\top\mathbf{W}\mathbf{x}^k$. 
Now, we are ready to design inner iterative strategy to update $\mathbf{x}^{k+1}$. First, the task-specific module $\mathcal{D}^k(\mathbf{x}^k)$ is designed as a denoiser operator. Then the candidate variable $\hat{\mathbf{x}}^{k+1}$ can be rewritten as $\hat{\x}^{k+1}=(1-\alpha)\x^k+\alpha\mathcal{D}^k(\mathbf{x}^k)$. If $\hat{\x}^{k+1}$ satisfies the error control condition~\eqref{eq:ek-condition}, we set  $\mathbf{x}^{k+1}=\mathcal{F}^k(\hat{\mathbf{x}}^{k+1})$, else decrease the weight $\alpha$ by $\rho\alpha$ to obtain the candidate $\hat{\mathbf{x}}^{k+1}$. 
Actually, if $\alpha=0$, we estimated $\hat{\mathbf{x}}^{k+1}$ by solving the following linear equation 
\begin{equation*}
\begin{array}{l}
\left( \mathbf{Q}^{\top}\mathbf{Q}+\mathbf{W}^\top\mathbf{W}+\beta\left(\nabla_h^{\top}\nabla_h+\nabla_v^{\top}\nabla_v\right)\right)\mathbf{x}^{k+1} =  \mathbf{Q}^{\top}\mathbf{b}+\mathbf{s}^k,
\end{array}
\end{equation*}
with PCG method as stated in Subsection~\ref{subsec:differentiable_updating} satisfying $\|\mathbf{e}_k(\hat{\x}^{k+1})\|_2^2\leq \eta\|\mathbf{e}_k(\hat{\x}^k)\|_2^2$ which is defined as in~\eqref{eq:etaCriterion}. Then, variables of $\mathbf{x}^{k+1}$, $\mathbf{u}^{k+1}$ and $\bm{\bm{\lambda }}^{k+1}$ are updated following Alg.~\ref{alg:GO-ADMM}.

\textbf{Rain Streaks Removal}: For rain streaks removal application, we reformulate this problem with unknown background $\mathbf{x}_b$ layer and rain streaks layer $\mathbf{x}_r$ by the setting function $l(\mathbf{Qx})=\frac{1}{2}\|\mathbf{Qx}-\mathbf{b}\|_2^2$ and $g(\mathbf{x})=\mu_1\|\nabla\mathbf{x}_b\|_1+\mu_2\|\mathbf{x}_r\|_1$, 
where $\mathbf{x}:=[\mathbf{x}_b;\mathbf{x}_r]$ and $\mathbf{b}$ denotes the rainy image. In this case, we set $\mathbf{Q}$ as block unit matrix, i.e., $\mathbf{Qx} = \mathbf{x}_b+\mathbf{x}_r$. By introducing two auxiliary variables $\mathbf{u}$ and $\mathbf{v}$, we reformulate Eq.~\eqref{eq:model} as follows
\begin{equation*}
\begin{array}{l}
\min\limits_{\mathbf{x}_b,\mathbf{x}_r} \frac{1}{2}\|\mathbf{Qx}-\mathbf{b}\|_2^2+\mu_1\|\nabla\mathbf{x}_b\|_1+\mu_2\|\mathbf{x}_r\|_1,\\
s.t., \nabla\mathbf{x}_b-\mathbf{u} = \mathbf{0}\ \text{and}\ \mathbf{x}_r-\mathbf{v} = \mathbf{0}.
\end{array}
\end{equation*}
By introducing the dual multipliers $\bm{\lambda}=[\bm{\lambda_1};\bm{\lambda}_2]$ and the penalty parameter $\beta>0$, generalized updates form of variables are summarized as follows
\begin{equation*}
\!\left\{
\begin{array}{l}
\!\mathbf{x}_b^{k+1}\!\in\!\arg\min\mathcal{L}_{\beta}(\mathbf{x}_b,\mathbf{x}_r^k,\mathbf{u}^k,\mathbf{v}^k,\bm{\bm{\lambda }}^k)+\frac{1}{2}\|\mathbf{W}(\mathbf{x}_b-\mathbf{x}_b^k)\|_2^2,\\
\!\mathbf{x}_r^{k+1}\!\in\!\arg\min\mathcal{L}_{\beta}(\mathbf{x}_b^k,\mathbf{x}_r,\mathbf{u}^k,\mathbf{v}^k,\bm{\bm{\lambda }}^k)+\frac{1}{2}\|\mathbf{W}(\mathbf{x}_r-\mathbf{x}_r^k)\|_2^2,\\
\!\mathbf{u}^{k+1}\!\in\!\arg\min\mathcal{L}_{\beta}(\mathbf{x}_b^{k+1},\mathbf{x}_r^{k+1},\mathbf{u},\mathbf{v}^k,\bm{\bm{\lambda }}^k),\\
\!\mathbf{v}^{k+1}\!\in\!\arg\min\mathcal{L}_{\beta}(\mathbf{x}_b^{k+1},\mathbf{x}_r^{k+1},\mathbf{u}^{k},\mathbf{v},\bm{\bm{\lambda }}^k),\\
\!\bm{\bm{\lambda }}_1^{k+1}\!=\!\bm{\bm{\lambda }}_1^{k}+\beta(\nabla\mathbf{x}_b^{k+1}-\mathbf{u}^{k+1}),\\
\!\bm{\bm{\lambda }}_2^{k+1}\!=\!\bm{\bm{\lambda }}_2^{k}+\beta(\mathbf{x}_r^{k+1}-\mathbf{v}^{k+1}).
\end{array}
\right.
\end{equation*}
Indeed, for $\mathbf{x}_b$ and $\mathbf{x}_r$-subproblems, we set $\hat{\x}_b^{k+1}=(1-\alpha)\x_b^k+\alpha\mathcal{D}^k_b(\mathbf{x}_b^k)$ and $\hat{\x}_r^{k+1}=(1-\alpha)\x_r^k+\alpha\mathcal{D}^k_r(\mathbf{x}_r^k)$ as the denoiser and rain streaks removal operators respectively with different input (i.e., $\mathbf{x}_b^k$ and $\mathbf{x}_r^k$). We then follow the form in image restoration applications to update $\hat{\mathbf{\mathbf{x}}}_b^{k+1}$ and $\hat{\mathbf{x}}_r^{k+1}$.

\begin{figure*}[t]
	%	\vskip 0.2in
	\centering
	\begin{tabular}{c@{\extracolsep{0.3em}}c@{\extracolsep{0.3em}}c@{\extracolsep{0.3em}}c@{\extracolsep{0.3em}}c}
		\includegraphics[height=0.143\textwidth]{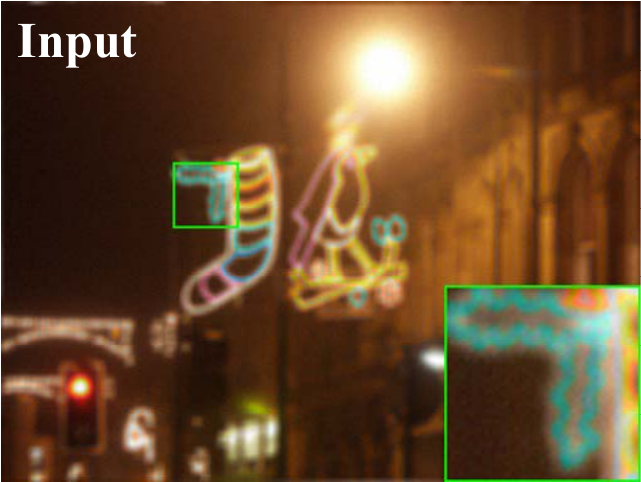}&
		\includegraphics[height=0.143\textwidth]{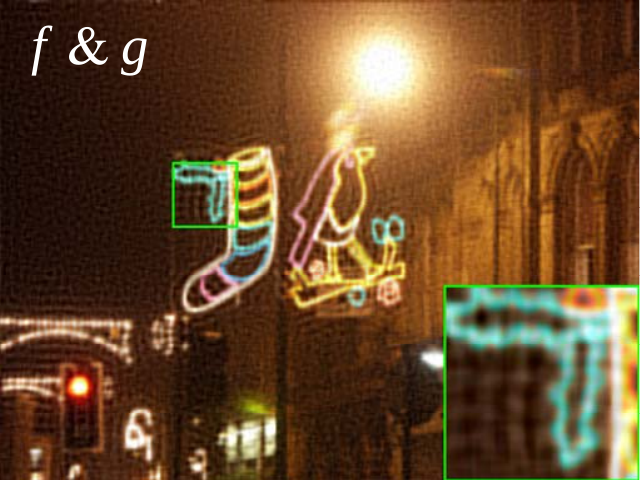}&
		\includegraphics[height=0.143\textwidth]{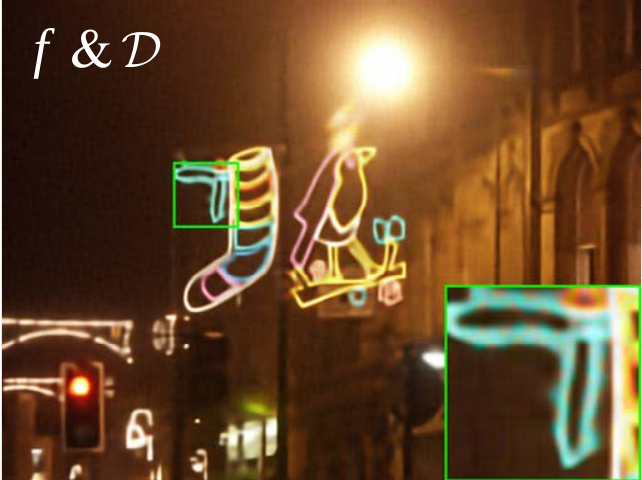}&
		\includegraphics[height=0.143\textwidth]{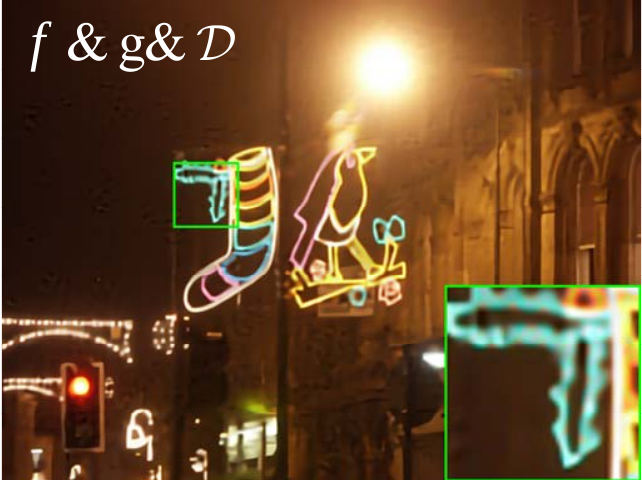}&
		\includegraphics[height=0.143\textwidth]{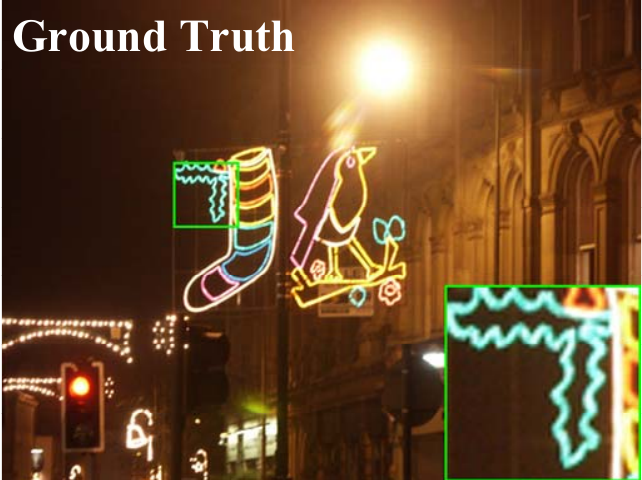}\\
	\end{tabular}
	\caption{Visual quality comparison of different models on image deblurring task. $f$, $g$ and $\mathcal{D}$ denote fidelity term, regularization/prior term and introduced data term respectively.}
	\label{fig:CompPriors}
\end{figure*}

\section{Numerical Results}\label{sec:exp}

This section first conducts experiments to verify our theoretical results. Then we compare the performance of the proposed algorithm scheme with other state-of-the-art methods in real-world problems, such as image denoising, inpainting, compressed sensing MRI, and rain streaks removal. All experiments are performed on a PC with Intel Core i7 CPU at 3.2 GHz, 32 GB RAM and a NVIDIA GeForce GTX 1070 GPU.

\subsection{Numerical Validation}\label{subsec:verification}
To verify the convergence properties of the proposed approach, this subsection is organized as follows: We first apply our algorithm on image denoising that aims to restore the latent grayscale or color image $\mathbf{x}$ from the corrupted observation $\mathbf{b}$ which relies on the linear model $\mathbf{b}=\mathbf{x}+\mathbf{n}$ to analyze the iterative behaviors and the error conditions with different $\mathcal{D}^k$ settings. In this application, we set $\mathbf{Q}$ as a unit matrix. Then, to make a further analysis, we illustrate the convergence behaviors of the proposed scheme on a general form. Here, we conduct experiments on image deconvolution, $\mathbf{b}=\mathbf{Q}\mathbf{x}+\mathbf{n}$, that aims to recover the unknown latent image $\mathbf{x}$ from the missing pixels of the observation $\mathbf{b}$ with the mask matrix $\mathbf{Q}$ and the noise $\mathbf{n}$. Detailed analyses are provided in the following.

We first compare the proposed method with standard ADMM in image denoising application with $20\%$ noise level and the corresponding metrics are plotted in Fig.~\ref{fig:CompNum}. To calculate the performance, we plotted the PSNR curves in Fig.~\ref{fig:CompNum}(a). Observed that the proposed GO-ADMM performed the best with fewer iteration steps when compared with standard ADMM. Indeed, the task-specific module is used to calculate a task-related optimal solution that plays a critical part in real-world applications. To further compare the convergence of GO-ADMM with other numerical ones, we plotted the iteration errors (i.e., $\sum_{i=1}^{3}\|\bm{w}^{k+1}(i)-\bm{w}^k(i)\|_2$ where $\bm{w}^k$ is defined in~\eqref{eq:wbar}) in Fig.~\ref{fig:CompNum}(b). Moreover, we provide a detailed analysis of the convergence by showing the errors about $\|\mathbf{x}^{k+1}-\mathbf{x}^k\|_2$, $\|\mathbf{y}^{k+1}-\mathbf{y}^k\|_2$ and $\|\mathbf{Ax}^{k+1}+\mathbf{By}^{k+1}-\mathbf{c}\|_2$ (i.e., plotted in Fig.~\ref{fig:CompNum}(c), (d), and (e) respectively). Note that, in image denoising applications, $\mathbf{A}$ is the gradient operator and $\mathbf{B}=\mathbf{I}_{l\times l}$ denotes the $l\times l$ identity matrix. Moreover, Fig.~\ref{fig:CompNum}(b)-(e) verified the convergence theories provided in Theorem~\ref{thm:Convergence}.

To illustrate the flexibility of the task-specific module in GO-ADMM, we then consider the performance of GO-ADMM with four different strategies for $\mathcal{D}^k$ settings, such as filtering (BM3D~\cite{dabov2007image}, RF~\cite{gastal2011domain}), discriminant learning (CSF)~\cite{schmidt2014shrinkage}, and the convolution neural networks (CNNs)~\cite{zhang2017learning}, in image denoising task under $20\%$ noise level. The corresponding convergence behaviors are plotted in the first four subfigures of Fig.~\ref{fig:CompNets} with reconstruction errors, relative errors ($\|\mathbf{x}^{k+1}-\mathbf{x}^{k}\|_2/\|\mathbf{x}^{k+1}\|_2$), and iterative errors of the dual variable (i.e., $\|\mathbf{Ax}^{k+1}+\mathbf{By}^{k+1}-\mathbf{c}\|_2$). These subfigures illustrate that different $\mathcal{D}^k$ settings have distinct effects on experimental performance and convergent behaviors. Observed that GO-ADMM with CNNs strategy performed the best than the others (i.e., with BM3D, CSF, RF settings). Consequently, we set the task-specific module as a set of CNNs in the following work. Additionally, it is crucial to show the guidance behaviors (i.e., the error control condition $\mathbf{e}_k(\hat{\mathbf{x}}^{k+1})$ and $\mathbf{e}_k(\hat{\mathbf{x}}^k)$ as described in Eq.~\eqref{eq:ek}) under CNNs setting when conduct the experiment. Then, the corresponding curves are plotted in Fig.~\ref{fig:CompNets} (e) in which $l_{\eta}$ denotes the lower bound of $\sqrt{2\theta}/(\sqrt{2\theta}+L\|\mathcal{N}\|_2)>l_{\eta}$. As for the estimation of $\|\mathcal{N}\|_2$, we have that $\|\mathcal{N}\|_2\leq\left\|(\mathbf{W}^\top\mathbf{W}+\beta\mathbf{A}^{\top}\mathbf{A})^{-1}\begin{pmatrix}\mathbf{W}^{\top}&\sqrt{\beta}\mathbf{A}^{\top}\end{pmatrix}\right\|_2=1/\sqrt{\lambda_{\min}}$, where $\lambda_{\min}$ is the minimum characteristic value of $(\mathbf{W}^\top\mathbf{W}\!+\!\beta\mathbf{A}^{\top}\mathbf{A})$ and $\mathbf{W}=\tau\mathbf{I}$ with $\tau>0$. Thus, we have $\lambda_{\min}>\tau$. In this experiment, parameters $\tau$ and $\mu$ are set as $\sqrt{2}$ and $1e-4$, respectively. With $\theta= L=1$, the relationship between $\|\mathcal{N}\|_2 < 1/\sqrt{2}$ and $l_{\eta} \leq 2/3$ is satisfied. Clearly, experimental results imply that iterations with CNNs module met the error control condition stated in Eq.~\eqref{eq:ek-condition} (i.e., $\|\mathbf{e}_k(\hat{\mathbf{x}}^{k+1})\|_2\leq l_{\eta}\|\mathbf{e}_k(\hat{\mathbf{x}}^k)\|_2$). Further, convergent behaviors in Fig.~\ref{fig:CompNets} shows that GO-ADMM has no limit on the task-specific module $\mathcal{D}^k$. 

As for $\mathcal{D}^k$, we actually train a series of CNNs on images with different noise levels or rain streaks in image derain task as our ``bank'' of modules. As for CNNs, we just adopt the standard residual architecture, consisting of nineteen layers, i.e., seven dilated convolutions with $3 \times 3$ filter size, six ReLu operations (plugged between each two convolution layers) and five batch normalizations (plugged between convolution and ReLU, except the first convolution layer). In the training phase of image denoising, deblurring and inpainting, we randomly sample 800 natural images from the ImageNet database~\cite{russakovsky2015imagenet} and add Gaussian noise with different noise levels. In rain streaks removal application, we select Rain100L~\cite{martin2001database} and Rain1400~\cite{fu2017removing} datasets to train the CNNs. We use ADAM with a weight decay of 0.0001 to optimize MSE loss to train our network. The learning rate is initially set as 0.001 and decayed by multiplying 0.1 at the 30th, 60th and 80th epochs. As for $\mathbf{W}$, in image denoising, inpainting and MRI applications, it is selected as a parameter. In rain streaks removal task, $\mathbf{W}$ is designed as rain streaks mask measured by CNN networks.

\begin{figure}[t]
	\centering
	\begin{tabular}{c@{\extracolsep{0.2em}}c}
		\includegraphics[height=0.192\textwidth]{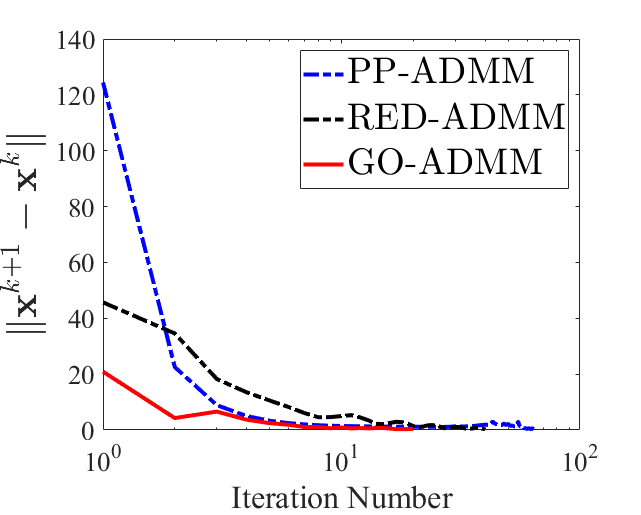}&
		\includegraphics[height=0.192\textwidth]{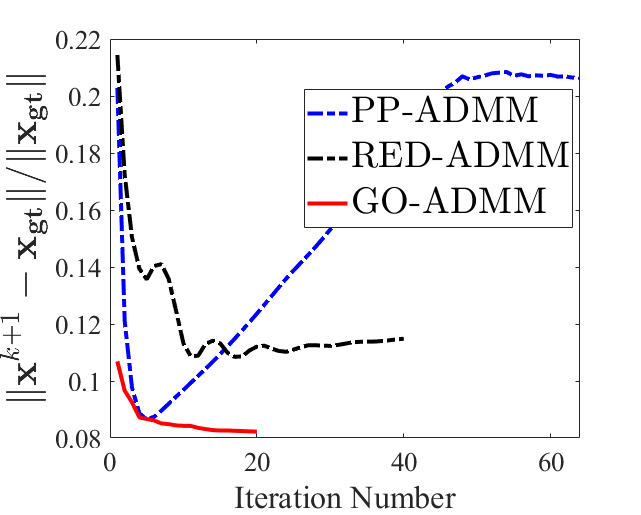}\\
	\end{tabular}
	\caption{Comparing iteration behaviors of GO-ADMM with PP-ADMM and RED-ADMM.}
	\label{fig:CompPP}
\end{figure}

To evaluate the effectiveness of regularization/prior knowledge and inserted data information, we conduct an ablation experiment on image deconvolution application with a $2\%$ noise level, and the corresponding results are shown in Fig.~\ref{fig:CompPriors}. In this experiment, $\mathbf{Q}$ represents a blurry kernel. Observed that, Fig.~\ref{fig:CompPriors} attributes the effectiveness of prior term and inserted data information. 

We further conduct an experiment on image deblurring task to compare the proposed method with implicit and explicit plug-in form (i.e., PP-ADMM and RED-ADMM). As shown in Fig.~\ref{fig:CompPP}, PP-ADMM and RED-ADMM curves are oscillating. This is mainly because the introduced denoiser $\mathcal{D}$ does not satisfy the strict smoothness and symmetrical structure required in RED-ADMM, and also does not meet the non-expansive requirement in PP-ADMM. Fortunately, GO-ADMM has no demand for the proposed task-specific module directly. In other words, the introduced guidance policy prevents the iteration sequence from tending to unwanted solutions.

\subsection{State-of-the-Art Comparisons}\label{sec:comparisons}

\textbf{Image Denoising.}
Image denoising aims to restore the latent  grayscale or color image $\mathbf{x}$ from the corrupted observation $\mathbf{b}$ that relies on the linear model $\mathbf{b} = \mathbf{x} + \mathbf{n}$. Here we compared our GO-ADMM with several state-of-the-art image restoration approaches, including plug-in methods (PP-ADMM~\cite{chan2017plug}, RED-ADMM~\cite{romano2017little}) and learning-based methods (DnCNN~\cite{zhang2017beyond}, CBDNet~\cite{guo2019toward}). We conducted experiments on the challenging real-world noisy image provided in~\cite{lebrun2015noise} and the corresponding results are shown in Fig.~\ref{fig:DenoiseReal}. Observed that our method removes more noise and restores a clearer image than the compared approaches.

\begin{figure*}[htb]
	\centering
	\begin{tabular}{c@{\extracolsep{0.3em}}c@{\extracolsep{0.3em}}c@{\extracolsep{0.3em}}c@{\extracolsep{0.3em}}c}
		\includegraphics[height=0.145\textwidth]{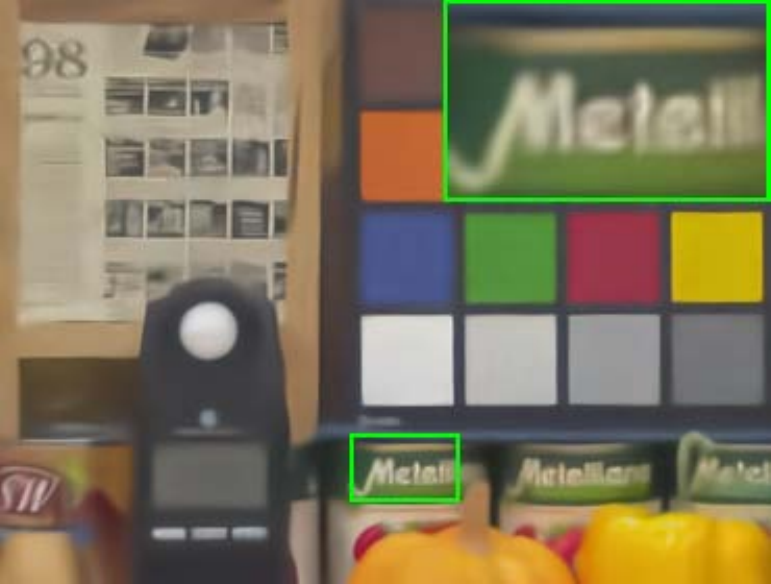}&
		\includegraphics[height=0.145\textwidth]{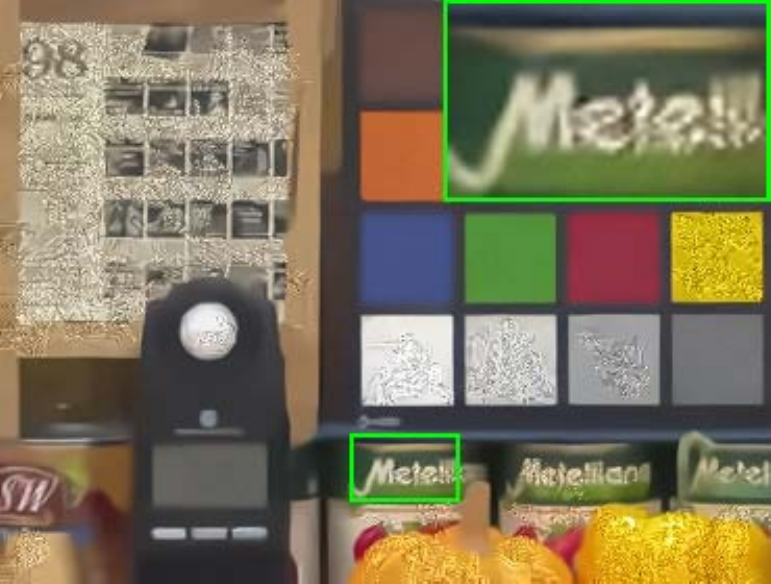}&
		\includegraphics[height=0.145\textwidth]{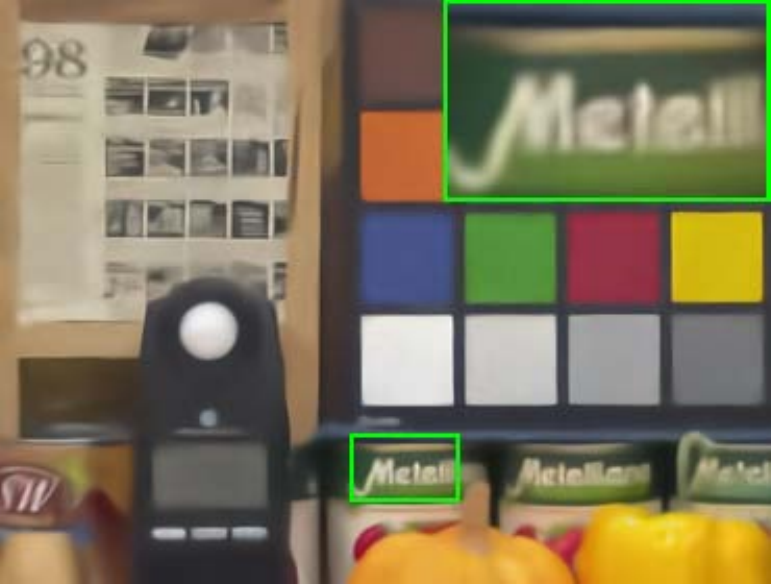}&
		\includegraphics[height=0.145\textwidth]{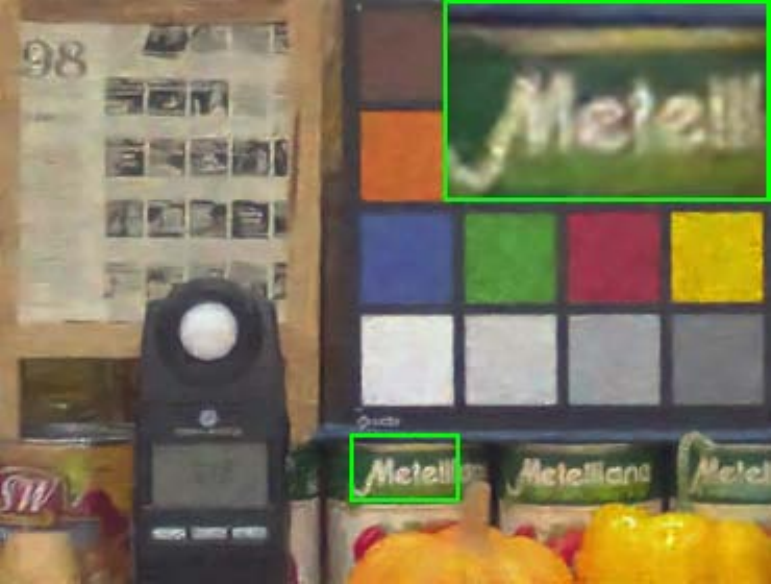}&
		\includegraphics[height=0.145\textwidth]{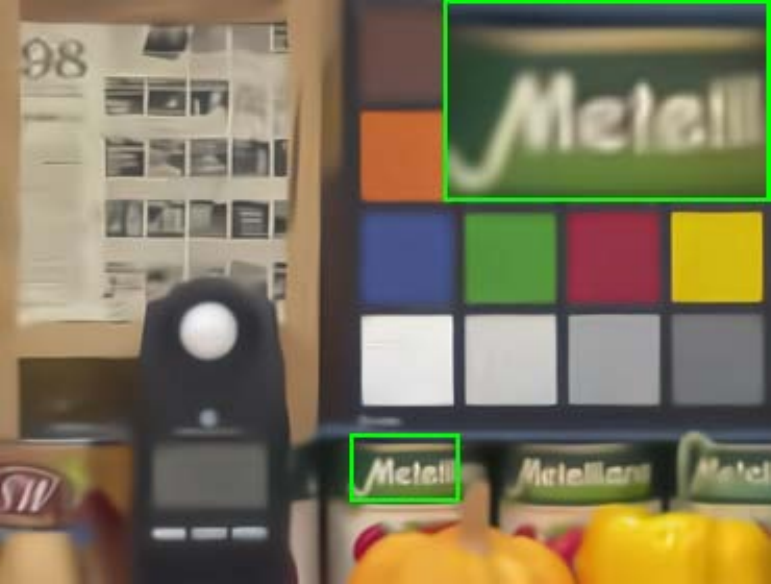}\\
		\footnotesize PP-ADMM &\footnotesize RED-ADMM  &\footnotesize DnCNN & \footnotesize CBDNet & \footnotesize Ours \\
	\end{tabular}
	\caption{Comparisons with the state-of-the-art methods on real-world noisy image.}
	\label{fig:DenoiseReal}
\end{figure*}

\begin{figure*}[t]
	%	\vskip 0.2in
	\centering
	\begin{tabular}{c@{\extracolsep{0.2em}}c@{\extracolsep{0.2em}}c@{\extracolsep{0.2em}}c@{\extracolsep{0.2em}}c@{\extracolsep{0.2em}}c}
		\includegraphics[height=0.24\textwidth]{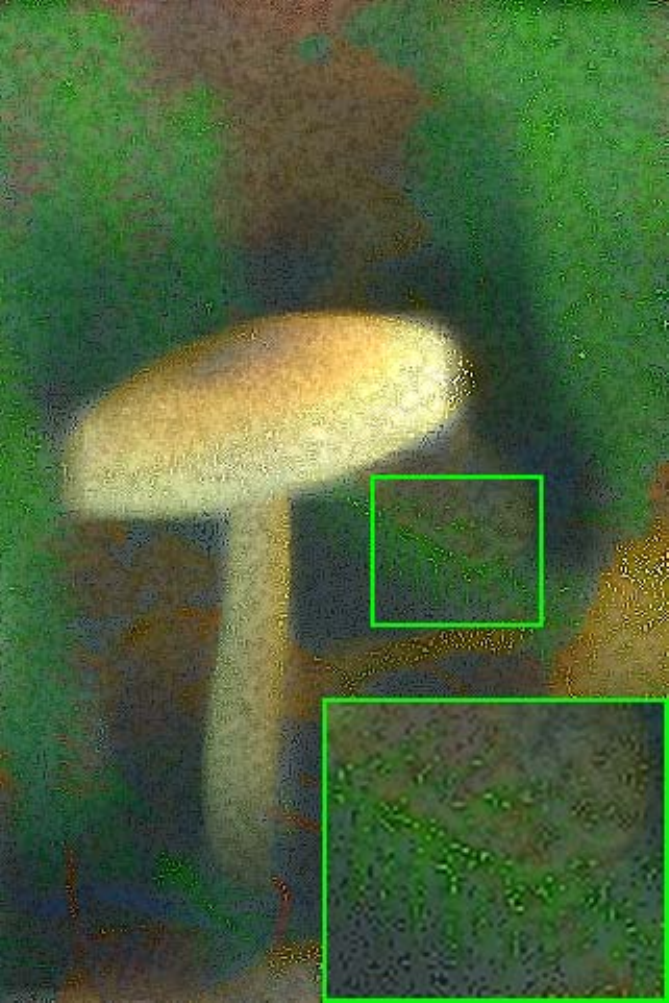}&
		\includegraphics[height=0.24\textwidth]{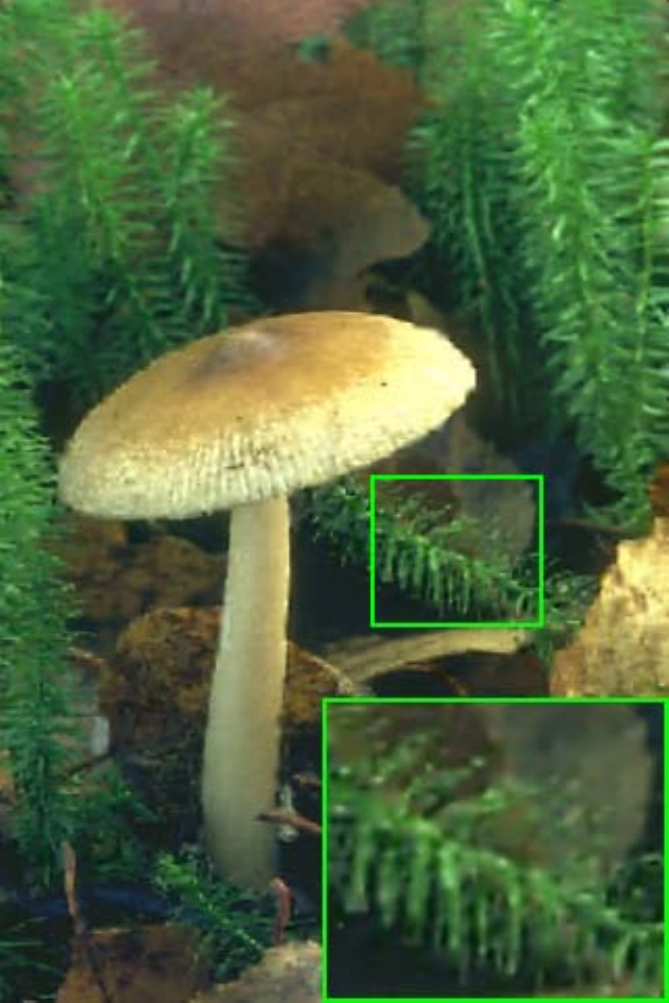}&
		\includegraphics[height=0.24\textwidth]{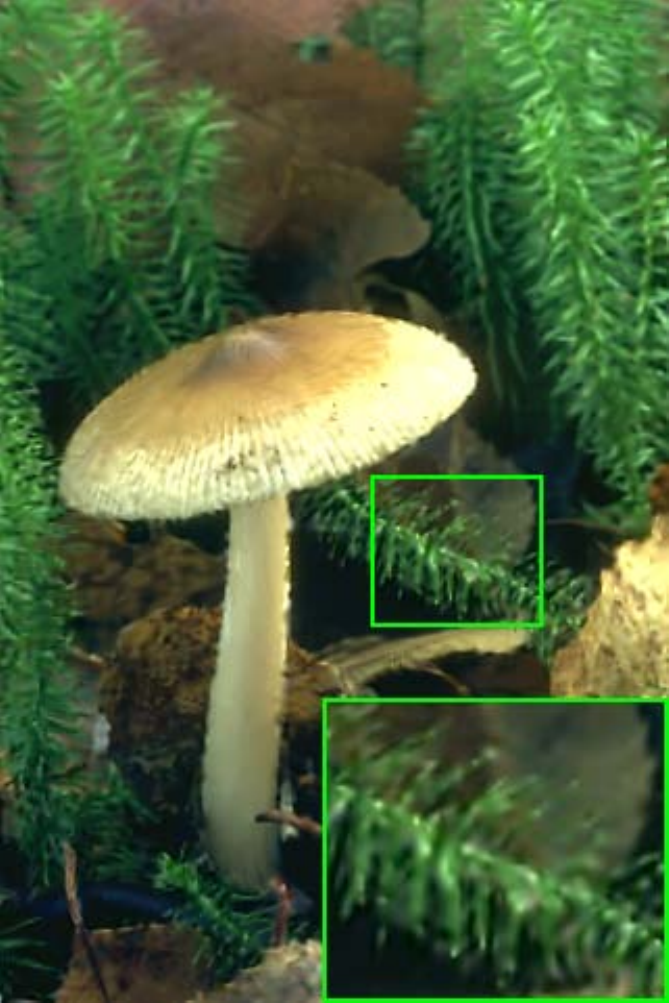}&
		\includegraphics[height=0.24\textwidth]{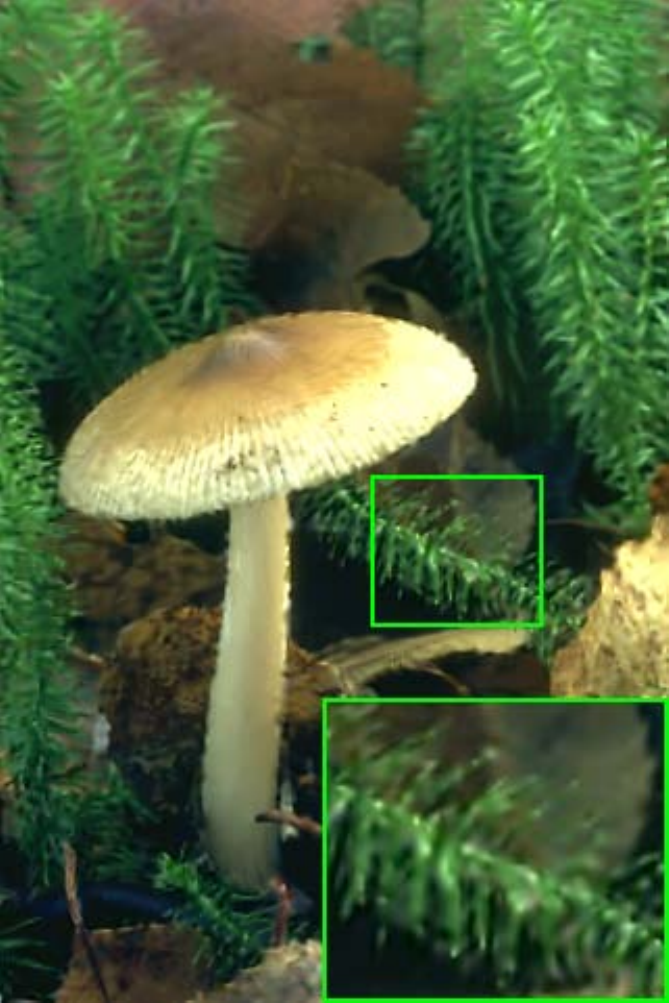}&
		\includegraphics[height=0.24\textwidth]{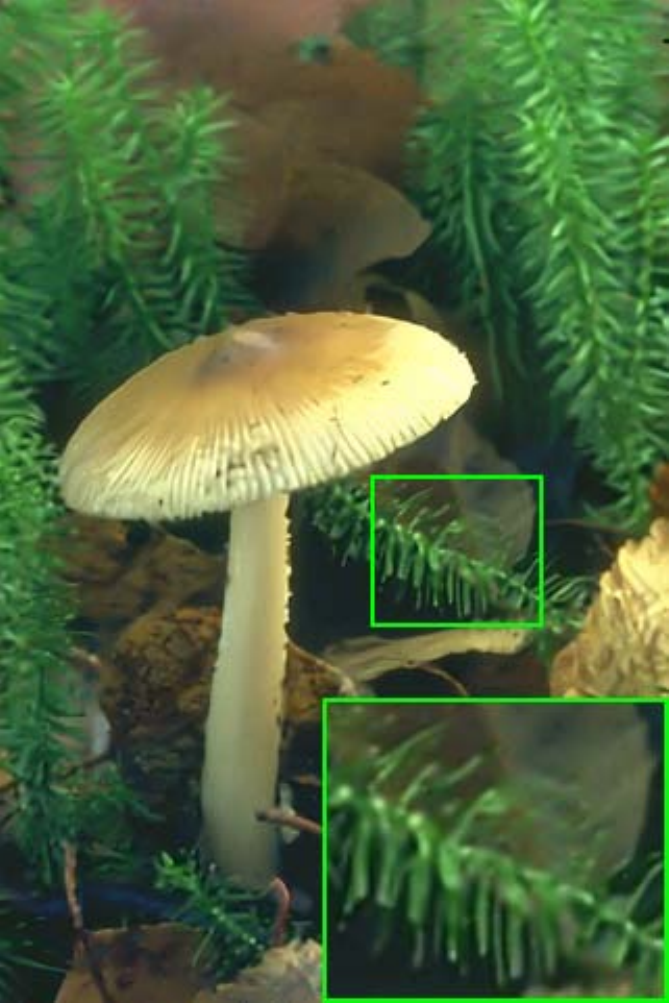}&
		\includegraphics[height=0.24\textwidth]{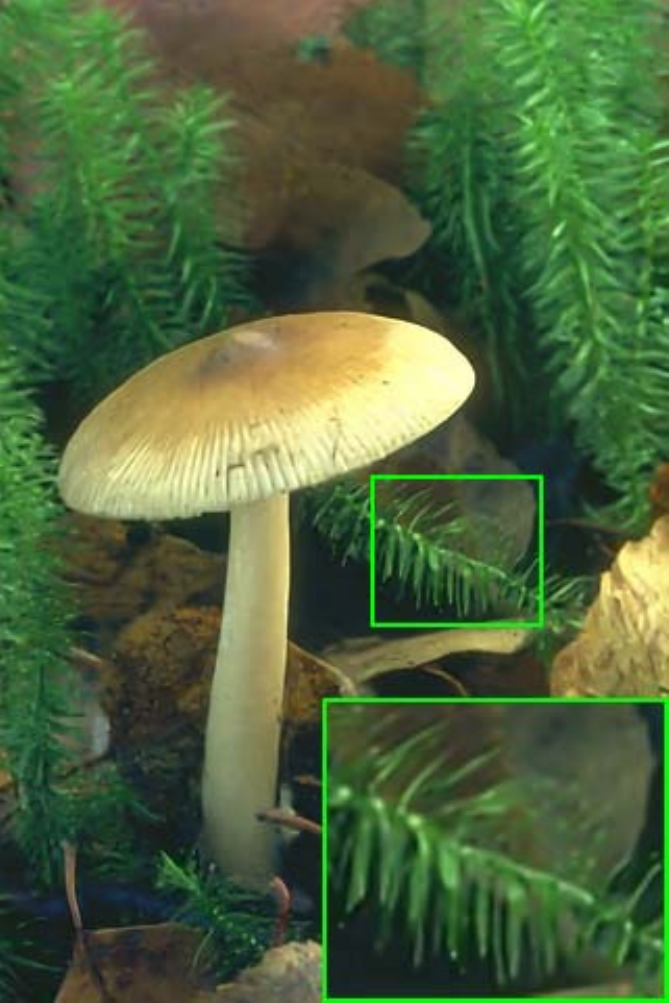}\\
		\footnotesize 19.86 / 0.32 &\footnotesize 28.77 / 0..82 &\footnotesize 26.69 / 0.82 &\footnotesize 28.67 / 0.80 &\footnotesize  28.89 / 0.82 &\footnotesize 29.54 / 0.84 \\
		\footnotesize ISDSB &\footnotesize FoE &\footnotesize 
		WNNM &\footnotesize IRCNN&\footnotesize LBS&\footnotesize Ours \\
	\end{tabular}
	\caption{Comparisons with the state-of-the-art methods on image inpainting with PSNR and SSIM scores.}
	\label{fig:Inpainting}
	%	\vskip -0.2in
\end{figure*}

\begin{table}[t]
	\renewcommand\arraystretch{1.1}
	\centering
	\caption{Averaged image completion performance with three different noise levels of missing pixels and text mask on CBSD68 dataset~\cite{zhang2017beyond}. The first column is the proportion of missing pixels and the text mask. The first column is the comparison methods on inpainting.}
	\label{tab:inpainting}
	\vspace{-0.2em}
	\setlength{\tabcolsep}{0.8mm}{
		\begin{center}
			\begin{small}
				\begin{tabular}{|p{1cm}<{\centering}|p{0.8cm}<{\centering}|p{0.9cm}<{\centering}|p{1.2cm}<{\centering}|p{1.2cm}<{\centering}|p{0.8cm}<{\centering}|p{1.5cm}<{\centering}|}%{c|cccccc}%
					\hline%\Xhline{1pt}
					Mask & FoE & ISDSB & WNNM & IRCNN &LBS& Ours\\
					\hline%\Xhline{1pt}
					\hline
					\multirow{2}{*}{$40\%$} & 34.01 & 31.32&31.75 & 34.92 & 34.54&\textbf{34.96}\\
					\cline{2-7} 
					&0.90  &0.91 & 0.94& 0.95& 0.95&\textbf{0.98}\\
					\hline
					\hline
					\multirow{2}{*}{$60\%$} & 30.81 &28.23 &28.71 &31.45 &31.27&\textbf{31.56}\\
					\cline{2-7} 
					&0.81  &0.83 & 0.89& \textbf{0.91}& 0.90&\textbf{0.91}\\
					\hline
					\hline
					\multirow{2}{*}{$80\%$} &27.64& 24.92 &25.63 &26.44 & 27.71&\textbf{27.89}\\
					\cline{2-7} 
					&0.65  &0.70 & 0.78& 0.79& 0.80&\textbf{0.81}\\
					\hline
					\hline
					\multirow{2}{*}{Text} & 37.05 &34.91 & 34.89&37.26& 36.88 &\textbf{37.47}\\
					\cline{2-7} 
					&0.95  &0.96 & 0.97& 0.97& 0.94&\textbf{0.98}\\
					\hline%\Xhline{1pt}
				\end{tabular}
			\end{small}
		\end{center}
	}
\end{table}

\begin{figure*}[t]
	\centering
	\begin{tabular}{c@{\extracolsep{0.2em}}c@{\extracolsep{0.2em}}c@{\extracolsep{0.2em}}c@{\extracolsep{0.2em}}c@{\extracolsep{0.2em}}c}
		\includegraphics[height=0.158\textwidth]{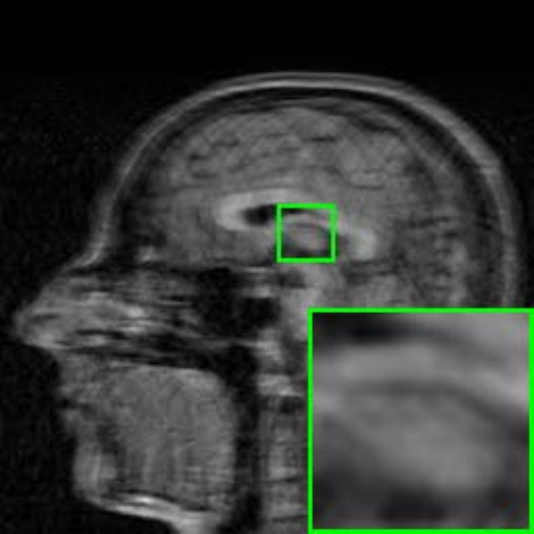}&
		\includegraphics[height=0.158\textwidth]{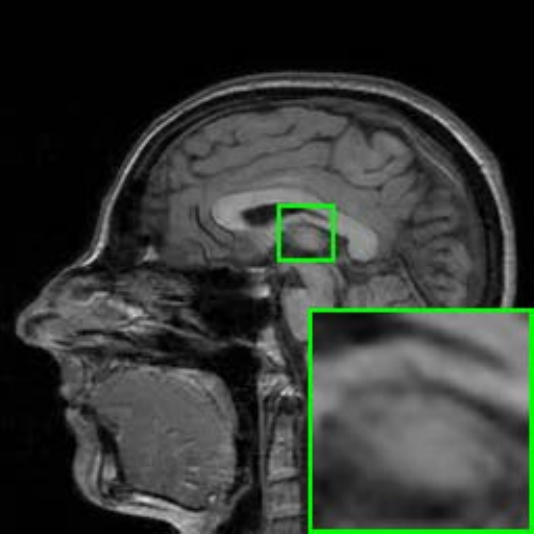}&
		\includegraphics[height=0.158\textwidth]{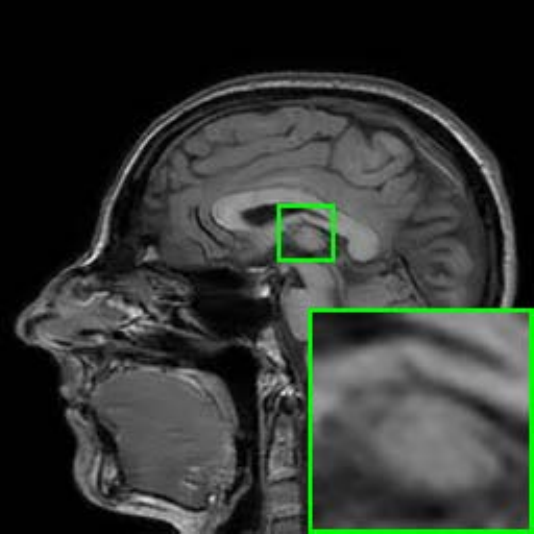}&
		\includegraphics[height=0.158\textwidth]{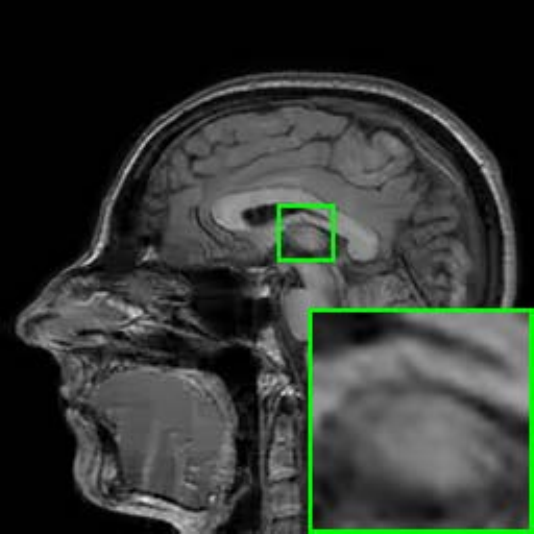}&
		\includegraphics[height=0.158\textwidth]{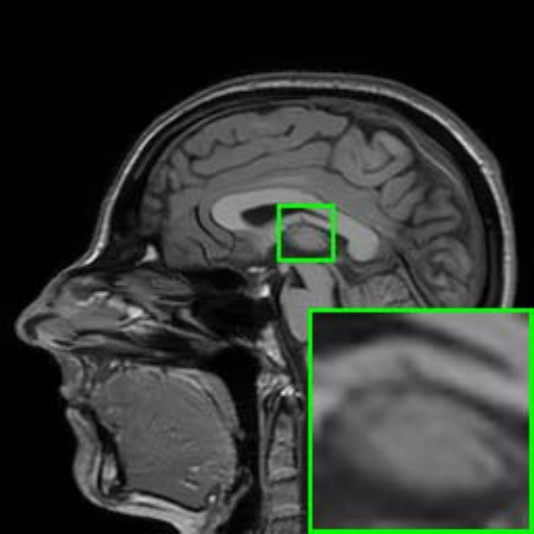}&
		\includegraphics[height=0.158\textwidth]{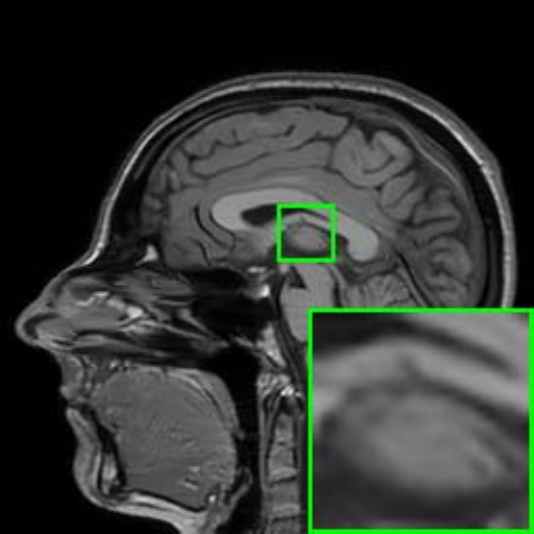}\\
		\footnotesize Input  &\footnotesize PANO (29.75) & \footnotesize FDLCP (29.99) &\footnotesize BM3D-MRI (28.97) 
		&\footnotesize TGDOF (32.06) &\footnotesize Ours (32.30)\\
	\end{tabular}
	\caption{Comparisons with the state-of-the-art methods on compressed sensing MRI with PSNR scores.}
	\label{fig:MRI}
\end{figure*}

\textbf{Image Inpainting.}
In image inpainting task, the matrix $\mathbf{Q}$ denotes mask, and $\mathbf{b}$ represents the missing pixels image. Then, we conducted experiments on image inpainting to recover missing pixels from observation. Here we compared our GO-ADMM with FOE~\cite{roth2009fields}, ISDSB~\cite{he2014iterative}, WNNM~\cite{li2009markov}, IRCNN~\cite{zhang2017learning}, and LBS~\cite{liu2018toward} on this task. We generated random masks of different levels (i.e., $40\%$, $60\%$, and $80\%$) and missing pixels on the CBSD68 dataset~\cite{zhang2017beyond}. Further, we use 12 different text masks to evaluate the developed approach. The comparison results are listed in Tab.~\ref{tab:inpainting} with averaged quantitative results (PSNR and SSIM scores). Regardless of the proportion of masks, our GO-ADMM can achieve better performance when compared with the other methods. Furthermore, the visual performance of $80\%$ missing pixels are shown in Fig.~\ref{fig:Inpainting}. It can be seen that the proposed method outperformed all the compared methods on both visualization and metrics (PSNR and SSIM).

\begin{table}[t]
	\renewcommand\arraystretch{1.1}
	\centering
	\caption{Comparison with different approaches on the compressed sensing MRI problem in three kind patterns at a unified sampling ratio of 30\%.}
	\label{tab:MRI}
	\vspace{-0.2em}
	\small
	\begin{center}
		\begin{small}
			\begin{tabular}{|p{1.65cm}<{\centering}|p{0.65cm}<{\centering}|p{0.75cm}<{\centering}|p{0.65cm}<{\centering}|p{0.75cm}<{\centering}|p{0.65cm}<{\centering}|p{0.75cm}<{\centering}|}
				\hline
				\multirow{2}{*}{Methods } & \multicolumn{2}{c|}{Cartesian}  & \multicolumn{2}{c|}{Radial} & \multicolumn{2}{c|}{Gaussian}\\
				\cline{2-7}
				&  PSNR & RLNE &  PSNR & RLNE &  PSNR & RLNE\\
				\hline%\Xhline{1pt}
				\hline
				ZeroFilling  & 23.95& 0.2338& 27.66  & 0.1524 &26.80& 0.1683 \\
				\hline
				TV & 26.30 & 0.1790 & 31.11 & 0.1032& 32.45 & 0.0885\\
				\hline
				SIDWT &25.77 & 0.1896 & 31.44 & 0.0994 & 32.49 & 0.0880\\
				\hline
				PANO  & 29.73 & 0.1206 & 33.49  & 0.0786 & 36.98 & 0.0527 \\
				\hline
				FDLCP	& 29.65 & 0.1218 & 34.31  & 0.0713 & 38.26 & 0.0452 \\
				\hline
				ADMMNet	 & 27.72 & 0.1520 & 33.31  & 0.0802 & 36.56 & 0.0550 \\
				\hline
				BM3D-MRI & 28.86 &	0.1330 & 34.59 & 0.0692 & 39.16 & 0.0410\\
				\hline
				TGDOF & 31.30 &	0.1007 & 35.47 & 0.0625  &	39.75 & 0.0382\\
				\hline
				Ours & \textbf{31.33} & \textbf{0.1003} & \textbf{35.51} & \textbf{0.0622}  &	\textbf{39.93} & \textbf{0.0374}\\
				\hline
			\end{tabular}
		\end{small}
	\end{center}
\end{table}

\begin{figure*}[t]
	\centering
	\begin{tabular}{c@{\extracolsep{0.2em}}c@{\extracolsep{0.2em}}c@{\extracolsep{0.2em}}c@{\extracolsep{0.2em}}c@{\extracolsep{0.2em}}c}
		\includegraphics[height=0.175\textwidth]{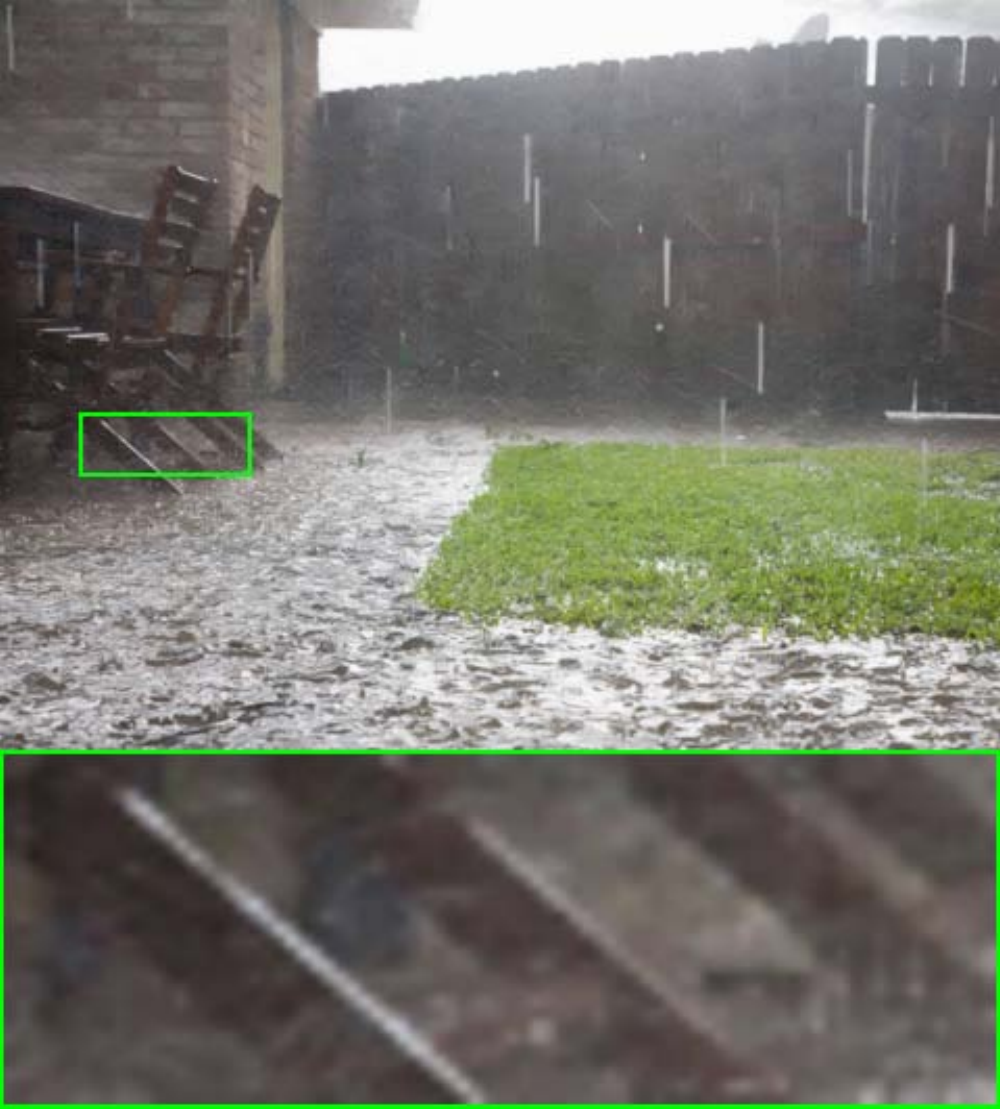}&
		\includegraphics[height=0.175\textwidth]{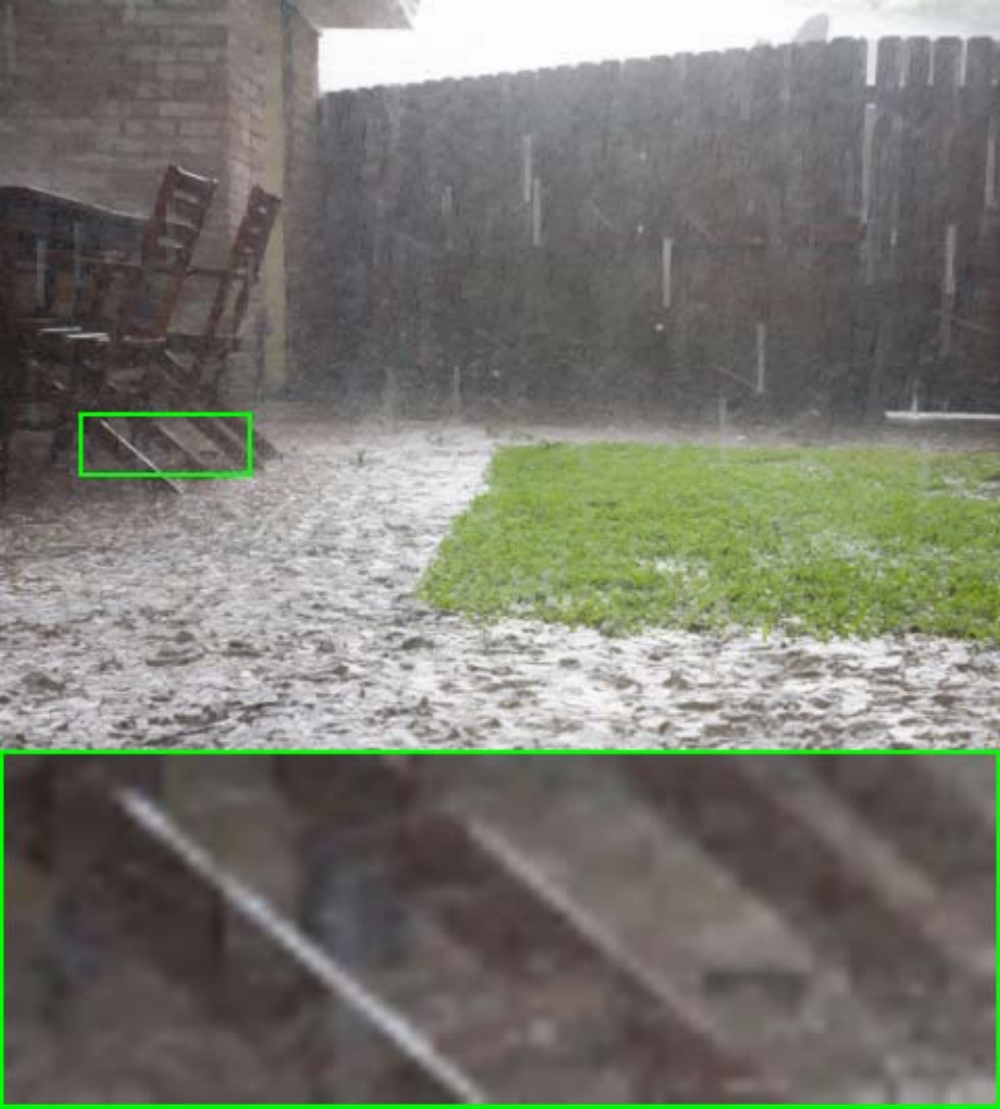}&
		\includegraphics[height=0.175\textwidth]{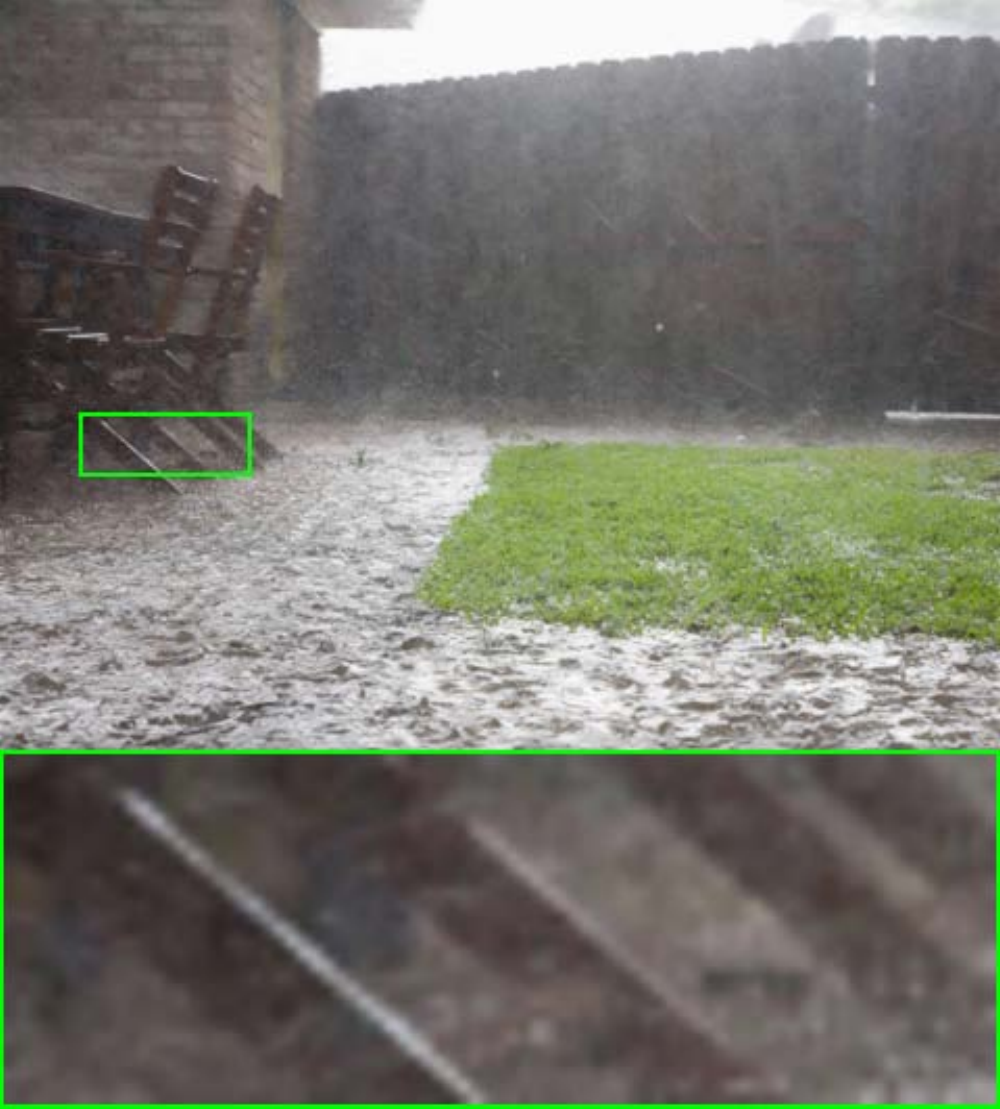}&
		\includegraphics[height=0.175\textwidth]{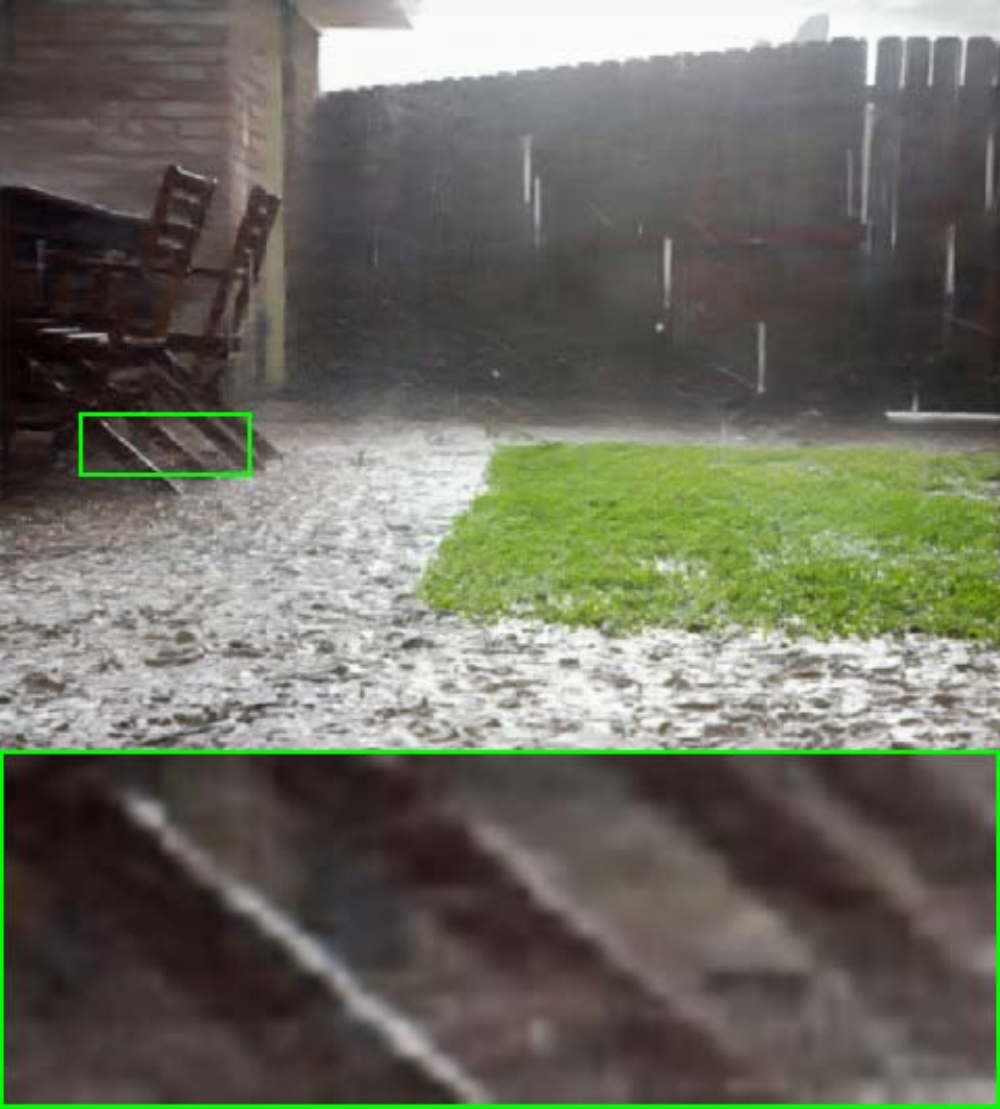}&
		\includegraphics[height=0.175\textwidth]{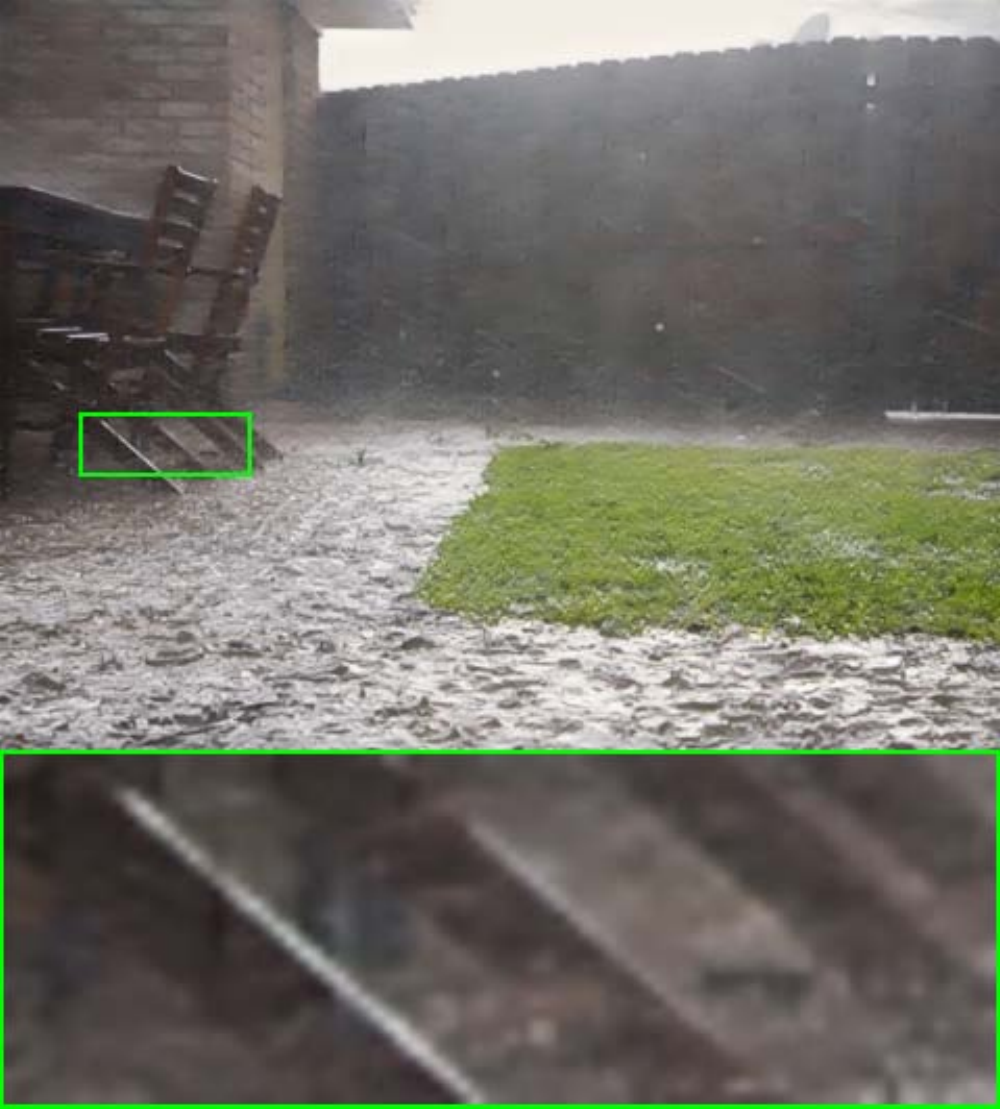}&
		\includegraphics[height=0.175\textwidth]{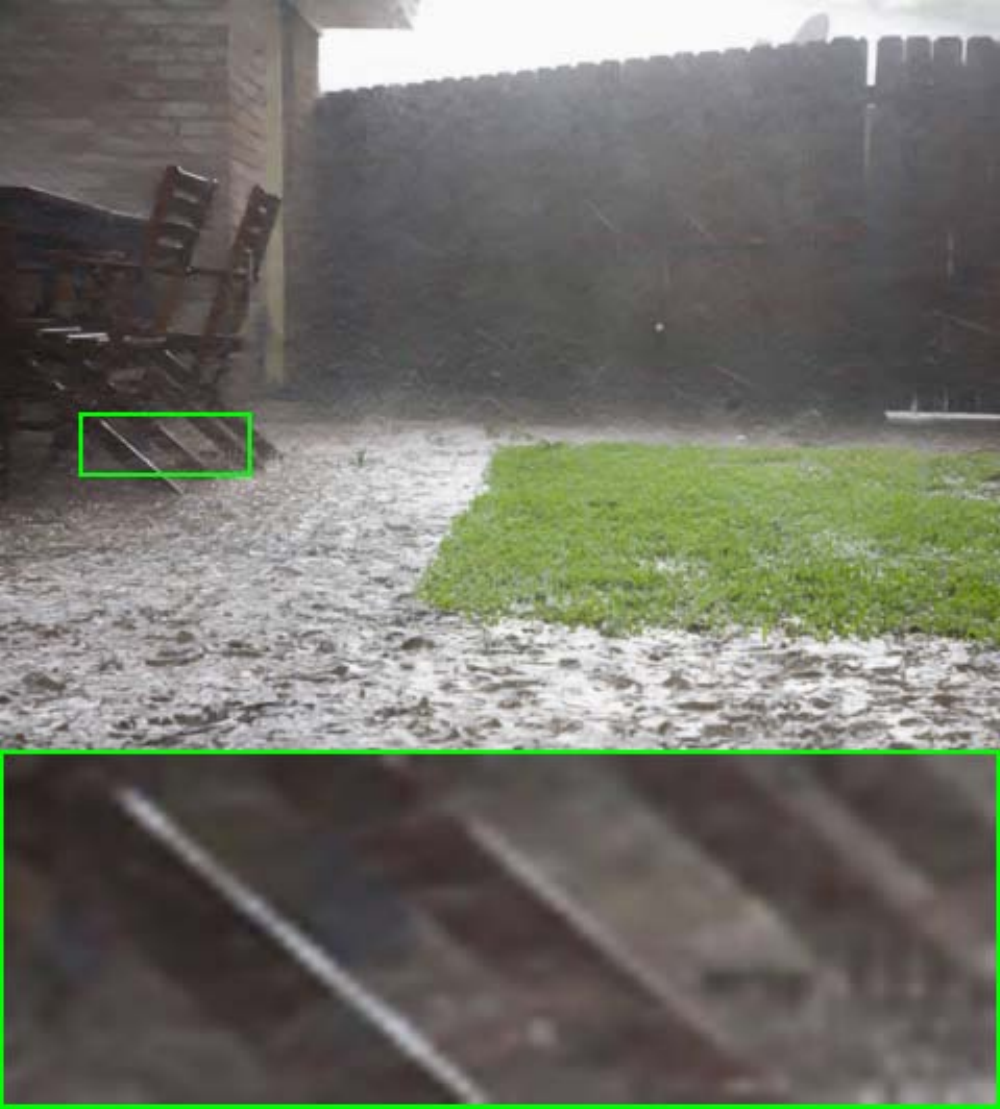}\\
		\footnotesize DDN &\footnotesize UGSM & \footnotesize JORDER&\footnotesize DID-MDN & \footnotesize PReNet  &\footnotesize Ours
	\end{tabular}
	\caption{Comparisons with the state-of-the-art methods on real-world rain streaks removal.}
	\label{fig:DerainReal}
\end{figure*}

\textbf{Compressed Sensing MRI.}
In compressed sensing Magnetic Resonance Imaging (MRI) task, we define $\mathbf{Q}$ as under-sampling matrix $\mathbf{P}$ and Fourier transformation $\mathbf{F}$~\cite{sun2016deep}, i.e., $\mathbf{Q}\mathbf{x}=\mathbf{PFx}$. We randomly choose 25 $T_1$-weighted MRI data from 50 different subjects in IXI datasets~\footnote{http://brain-development.org/ixi-dataset/} as the testing data for our comparison. Then we adopt three types of sampling masks, i.e., Cartesian pattern~\cite{qu2012undersampled}, Radial pattern~\cite{sun2016deep}, and Gaussian mask~\cite{yang2018dagan} on selected $T_1$-weighted dataset. We compare our GO-ADMM with ZeroFilling~\cite{bernstein2001effect}, TV~\cite{lustig2007sparse}, SIDWT~\cite{baraniuk2007compressive}, PANO~\cite{qu2014magnetic}, FDLCP~\cite{zhan2016fast}, ADMMNet~\cite{sun2016deep}, BM3D-MRI~\cite{eksioglu2016decoupled}, and TGDOF~\cite{liu2019converged}, and the comparison results are shown in Tab.~\ref{tab:MRI}. Observed that, our paradigm shows great superiority in reconstruction accuracy and has a better ability to accommodate sampling patterns. We then provide the visual comparisons in Fig.~\ref{fig:MRI} for three methods with relatively high PSNR and RLNE scores (i.e., PANO, FDLCP, and TGDOF). Consistently, our GO-ADMM achieves the best performance in terms of both the restoration of detail and the PSNR scores.

\begin{table}[t]
	\renewcommand\arraystretch{1.1}
	\centering
	\caption{Averaged PSNR and SSIM results among different rain streaks removal methdos on three different rain streaks synthesized form.}
	\vspace{-0.2em}
	\small
	\begin{tabular}{|p{1.52cm}<{\centering}|p{0.7cm}<{\centering}|p{0.7cm}<{\centering}|p{0.7cm}<{\centering}|p{0.7cm}<{\centering}|p{0.7cm}<{\centering}|p{0.7cm}<{\centering}|}
		\hline
		\multirow{2}*{Methods} &\multicolumn{2}{c|}{Rain100L} &\multicolumn{2}{c|}{Rain1400} %Fu \emph{et al.}'
		&\multicolumn{2}{c|}{Rain100H}\\
		\cline{2-7}
		& PSNR & SSIM & PSNR & SSIM & PSNR & SSIM\\
		\hline\hline
		DN    & 27.27 & 0.8746 & 25.51 & 0.8885 & 13.72 & 0.4417\\
		\hline
		DDN   & 29.73 & 0.9177 & 29.90 & 0.8999 & 17.93 & 0.5655\\
		\hline
		UGSM  & 28.59 & 0.8772 & 26.38 & 0.8261 & 14.90 & 0.4674\\
		\hline
		JORDER &36.61 & 0.9722 & 27.50 & 0.8515 & 23.45 & 0.7490\\
		\hline
		DID-MDN & 25.45 & 0.8550 & 27.94 & 0.8696 & 17.28 & 0.6035\\
		\hline
		PReNet & 37.63 & 0.9792 & 32.32 & 0.9320 & 29.46 & 0.8990 \\
		\hline%\hline
		Ours & 37.64 & 0.9792 & 32.57 & 0.9319 & 29.50 & 0.8990\\
		\hline
		Ours (w) & \textbf{37.76} &\textbf{0.9798} & \textbf{32.64 }& \textbf{0.9320} & \textbf{29.51} & \textbf{0.8991}\\
		\hline
	\end{tabular}%}
	\label{tab:derain}
\end{table}

\textbf{Rain Streaks Removal.} To evaluate the performance of our method, we use both synthetic test data and real-world images to compare our approach with state-of-the-art approaches (including DerainNet (DN)~\cite{fu2017clearing}, Deep Detailed Network (DDN)~\cite{fu2017removing}, JORDER~\cite{yang2017deep}, UGSM~\cite{Jiang2018FastDeRain}, DID-MDN~\cite{zhang2018density}, and PReNet~\cite{ren2019progressive}) removing rain from single images. For measuring the performance quantitatively, we employ PSNR, SSIM as the metrics. All the comparisons shown in this paper are conducted under the same hardware configuration.

Tab.~\ref{tab:derain} shows the results of different methods on  Rain100L, Rain100H~\cite{martin2001database}, and Rain1400~\cite{fu2017removing} datasets. As observed, our method considerably outperforms others in terms of both PSNR and SSIM. Fig.~\ref{fig:DerainReal} compares the visual performance of GO-ADMM to the first five scores methods listed in Tab.~\ref{tab:derain} on real-world challenging rainy image. Observed that DDN tends to retain excessive rain streaks while GMM tends to keep rain streaks for images with over-smooth background details. Qualitatively, the proposed GO-ADMM achieves the best visual results in terms of effectively removing the rain streaks while preserving the scene details. 
We further evaluate the performance with designed $\mathbf{W}$. Here, $\mathbf{W}$ are selected as rain streaks mask measured by CNN networks~\cite{zhang2017learning}. We also listed the GO-ADMM results with designed $\mathbf{W}$ (i.e., GO-ADMM (w)) in Tab.~\ref{tab:derain}.

\section{Conclusions}\label{sec:conclusion}

In this paper, we proposed a collaborative learning scheme with the task-specific flexible module for specific optimization problems to solve vision and learning tasks. We provided strict theoretical analysis for the proposed GO-ADMM by introducing an error based criterion condition to measure the inexactness of the inner iterations. Further, the experimental results verified that GO-ADMM can even obtain better performance against most other state-of-the-art approaches.

\appendices

\section{Proofs}

We first provide some preparations. 
Since we just adopt standard computation strategy in Eqs.~\eqref{eq:y-iter}-\eqref{eq:lambda-iter} to update $\mathbf{y}$- and $\bm{\lambda }$-subproblems, the first-order optimality condition for sequence $\{\bm{w}^{k}\}$ generated by our GO-ADMM can be summarized as
\begin{equation}\label{eq:optimality}
\left\{
\begin{aligned}
& \nabla_{\mathbf{x}} {\mathcal{L}}_{\beta}( \mathbf{x}^{k+1},\mathbf{y}^k,\bm{\lambda }^k )+\bar{\mathbf{W}}(\mathbf{x}^{k+1}-\mathbf{x}^k)=\mathbf{Q}^\top \mathbf{e}^k (\hat{\mathbf{x}}^{k+1}) ,\\
& g(\mathbf{y})-g (\mathbf{y}^{k+1}) + (\mathbf{y}-\mathbf{y}^{k+1})^\top (-\mathbf{B}^\top \bm{\lambda }^k\\
&+\beta \mathbf{B}^\top (\mathbf{A}\mathbf{x}^{k+1}+\mathbf{B}\mathbf{y}^{k+1}-\mathbf{c}))\geq 0, \forall \mathbf{y}\in {\mathbb{R}}^m,\\
&\bm{\lambda }^{k+1}=\bm{\lambda }^{k}-\beta(\mathbf{A}\mathbf{x}^{k+1}+\mathbf{B}\mathbf{y}^{k+1}-\mathbf{c}).
\end{aligned}
\right.
\end{equation}
For convenience, we denote $\Lambda^{i,j}:=\|\bm{w}^i-\bm{w}^{j}\|_{\mathbf{M}}^2$,  $\Lambda^{i,j}_{1/2}=\|\bm{w}^i-\bm{w}^{j}\|_{\mathbf{M}}$ and $\Lambda^{j}:=\|\bm{w}-\bm{w}^{j}\|_{\mathbf{M}}^2$.

\subsection{Proof of the Proposition~\ref{prop:ek_relation}}% Proposition 1.

\begin{proof}
	Recall that the $\mathbf{y}$- and $\bm{\lambda }$-subproblems are assumed to be solved exactly in the GO-ADMM. We thus know that $\mathbf{B}^\top \bm{\lambda }^{k-1}\in \partial g(\mathbf{y}^{k-1}),  \mathbf{B}^\top \bm{\lambda }^{k}\in \partial g(\mathbf{y}^{k}),$ 
	and thus
	\begin{equation}\label{Basic_equation_00}
	\begin{array}{l}
	( \mathbf{y}^{k-1} - \mathbf{y}^{k} )^\top  \mathbf{B}^\top  ( \bm{\lambda }^{k-1} - \bm{\lambda }^{k} )\geq 0.
	\end{array}
	\end{equation}
	Hence, it is easily derived that
	\begin{equation}\label{Basic_equation_0}
	\begin{array}{l}
	\quad\| -\sqrt{\beta} \mathbf{B}(\mathbf{y}^{k-1}-\mathbf{y}^k)+\frac{1}{\sqrt{\beta}}(\bm{\lambda }^{k-1}-\bm{\lambda }^k) \|_2^2 \\
	\leq\frac{1}{\beta}\| \bm{\lambda }^{k-1} - \bm{\lambda }^{k} \|_2^2 + \beta  \| \mathbf{B} ( \mathbf{y}^{k-1} - \mathbf{y}^{k} )\|_2^2.
	\end{array}
	\end{equation}
	With the definition of $\mathbf{e}_k$ in Eq.~\eqref{eq:ek}, we have
	\begin{equation}\label{new-ek}
	\begin{array}{l}
	\| \mathbf{e}_k (\hat{\mathbf{x}}^{k+1}) \|_2
	\le  \eta \| \mathbf{e}_{k-1}(\hat{\mathbf{x}}^{k}) \|_2 + \eta \| \mathbf{e}_{k-1}(\hat{\mathbf{x}}^{k})-\mathbf{e}_{k}(\hat{\mathbf{x}}^{k})\|_2 \nonumber\\
	\overset{\eqref{eq:fLipschitz}}{\le} 
	\eta \| \mathbf{e}_{k-1} (\hat{\mathbf{x}}^{k})  \|_2 + \eta L \| \mathbf{Q}(\bar{\mathbf{W}}+\beta \mathbf{A}^\top \mathbf{A})^{-1}(\mathbf{s}^{k-1}-\mathbf{s}^k)\|_2 \nonumber\\
	\le \eta \| \mathbf{e}_{k-1} (\hat{\mathbf{x}}^{k})   \|_2\nonumber\\
	+\eta L \cdot \left\| {\mathcal{N}}\begin{pmatrix} \mathbf{W}(\mathbf{x}^{k-1}-\mathbf{x}^k)\\ -\sqrt{\beta} \mathbf{B}(\mathbf{y}^{k-1}-\mathbf{y}^k)+\frac{1}{\sqrt{\beta}}(\bm{\lambda }^{k-1}-\bm{\lambda }^k) \end{pmatrix}\right\|_2\nonumber\\
	\overset{(\ref{Basic_equation_0})}{\le}  \eta \| \mathbf{e}_{k-1} (\hat{\mathbf{x}}^{k})   \|_2 + \eta L \| \mathcal{N}\|_2 \Lambda^{k-1,k}_{1/2},
	\end{array}
	\end{equation}
	where $\bar{\mathbf{W}}=\mathbf{W}^\top\mathbf{W}$ and $\Lambda^{k-1,k}_{1/2}= \| \bm{w}^{k-1} - \bm{w}^k \|_{\mathbf{M}}$. 
	Therefore, it follows that 
	\begin{equation*}
	\| \mathbf{e}_k (\hat{x}^{k+1})  \|_2 \le \eta \| \mathbf{e}_{k-1} (\hat{x}^{k})  \|_2 + \eta\gamma \| \bm{w}^{k-1} - \bm{w}^k \|_{\mathbf{M}},
	\end{equation*}
	with $\gamma = L \| {\cal{N}}\|_2$. This relationship will be often used in the coming analysis. This complete the proof.
\end{proof}

\subsection{Proof of the Proposition~\ref{prop:re_Var_Inequality}}

\begin{proof}
	First we rewrite $\nabla_{\mathbf{x}} {\cal{L}}_{\beta} ( \mathbf{x}^{k+1},\mathbf{y}^k,\bm{\lambda }^k )$ as
	$$
	\quad\nabla_{\mathbf{x}} {\cal{L}}_{\beta}( \mathbf{x}^{k+1}, \mathbf{y}^k, \bm{\lambda }^k ) 
	= \mathbf{Q}^\top  \left( \nabla l(\mathbf{Q} \mathbf{x}^{k+1}) \right) - \mathbf{A}^\top  \bar{\bm{\lambda }}^k.\nonumber
	$$
	Then for all $\bm{w} \in\mathbf{\Omega}$, the following holds	
	\begin{equation}
	\begin{array}{l}
	V(\mathbf{y},\bar{\mathbf{y}}^k,\bm{w},\bar{\bm{w}}^k)
	= ( \mathbf{x} - \mathbf{x}^{k+1} )^\top  (\mathbf{Q}^\top \nabla l(\mathbf{Q} \mathbf{x}^{k+1}) - \mathbf{A}^\top  \bar{\bm{\lambda }}^k ) \nonumber\\
	+g(\mathbf{y}) - g(\mathbf{y}^{k+1}) + ( \mathbf{y}- \mathbf{y}^{k+1} )^\top  ( - \mathbf{B}^\top  \bar{\bm{\lambda }}^k ) \nonumber\\
	+( \bm{\lambda } - \bar{\bm{\lambda }}^k ) (\mathbf{A} \mathbf{x}^{k+1} + \mathbf{B} \mathbf{y}^{k+1} - \mathbf{c})\nonumber\\		
	= ( \mathbf{x} - \mathbf{x}^{k+1} )^\top  \nabla_{\mathbf{x}} {\cal{L}}_{\beta} (    \mathbf{x}^{k+1},\mathbf{y}^k,\bm{\lambda }^k )+g(\mathbf{y}) - g(\mathbf{y}^{k+1}) \nonumber\\+
	( \mathbf{y}- \mathbf{y}^{k+1}) \left[ -\mathbf{B}^\top  \bm{\lambda }^k + \beta \mathbf{B}^\top  ( \mathbf{A} \mathbf{x}^{k+1} + \mathbf{B} \mathbf{y}^{k+1} - \mathbf{c} ) \right] \nonumber\\
	+ \beta ( \mathbf{y}- \mathbf{y}^{k+1} ) \mathbf{B}^\top  \mathbf{B} ( \mathbf{y}^{k} - \mathbf{y}^{k+1} )	\!+\!\frac{1}{\beta} ( \bm{\lambda } \!-\! \bar{\bm{\lambda }}^k )^\top  ( \bm{\lambda }^k \!-\! \bm{\lambda }^{k+1} ).
	\end{array}
	\end{equation}
	where $V(\mathbf{y},\bar{\mathbf{y}}^k,\bm{w},\bar{\bm{w}})=g(\mathbf{y})-g(\bar{\mathbf{y}}^k)+ ( \bm{w}-\bar{\bm{w}} )^\top  \mathbf{F}(\bar{\bm{w}})$. 
	Combining it with the optimality \eqref{eq:optimality}, the above inequality yields that
	\begin{equation}
	\begin{array}{l}
	\!\!V(\mathbf{y},\bar{\mathbf{y}}^k,\bm{w},\bar{\bm{w}}^k)\\
	\!\!\ge ( \mathbf{x} - \mathbf{x}^{k+1})^\top (\mathbf{Q}^\top \mathbf{e}_k(\hat{\mathbf{x}}^{k+1})-
	\bar{\mathbf{W}}(\mathbf{x}^{k+1}-\mathbf{x}^k)) \nonumber\\
	\!\!+ \beta ( \mathbf{y}\!-\! \mathbf{y}^{k+1} ) \mathbf{B}^\top  \mathbf{B} ( \mathbf{y}^{k} \!-\! \mathbf{y}^{k+1} )  \nonumber\\
	\!\!+\frac{1}{\beta} ( \bm{\lambda } \!-\! \bm{\lambda }^{k+1} )^\top  ( \bm{\lambda }^k \!-\! \bm{\lambda }^{k+1} ) +\frac{1}{\beta} ( \bm{\lambda }^{k+1} \!-\! \bar{\bm{\lambda }}^k )^\top  ( \bm{\lambda }^k \!-\! \bm{\lambda }^{k+1} ).
	\end{array}	
	\end{equation}
	Moreover, notice the elementary equation
	\begin{equation}\label{Basic_equation_2}
	( \mathbf{a} - \mathbf{c} )^\top  ( \mathbf{b} - \mathbf{c} ) = \frac{1}{2} ( \| \mathbf{a} - \mathbf{c} \|_2^2 - \| \mathbf{a} - \mathbf{b} \|_2^2 + \| \mathbf{b} - \mathbf{c} \|_2^2 ).
	\end{equation}
	We denote the right hand of Eq.~\eqref{Basic_equation_2} as $\Delta_2(\mathbf{a},\mathbf{b},\mathbf{c})$.
	Thus, for all $\bm{w} \in \mathbf{\Omega}$, we have
	\begin{equation}
	\begin{array}{l}
	V(\mathbf{y},\bar{\mathbf{y}}^k,\bm{w},\bar{\bm{w}}^k) \nonumber \\
	\overset{(\ref{Basic_equation_2})}{\ge} ( \mathbf{x} - \mathbf{x}^{k+1} )^\top  \mathbf{Q}^\top \mathbf{e}_k(\hat{\mathbf{x}}^{k+1})+ \Delta_{\bar{\mathbf{W}}} (\mathbf{x},\mathbf{x}^k,\mathbf{x}^{k+1})	\nonumber\\
	+\beta\Delta_2(\mathbf{By},\mathbf{By}^k,\mathbf{By}^{k+1})
	+ \frac{1}{\beta} \Delta_2(\bm{\lambda},\bm{\lambda}^k,\bm{\lambda}^{k+1}) \nonumber\\+\frac{1}{\beta} ( \bm{\lambda }^{k+1} - \bar{\bm{\lambda }}^k )^\top  ( \bm{\lambda }^k - \bm{\lambda }^{k+1} )\\
	\overset{(\ref{Basic_equation_00})}{\ge} ( \mathbf{x} - \mathbf{x}^{k+1} )^\top   \mathbf{Q}^\top \mathbf{e}_k(\hat{\mathbf{x}}^{k+1})
	+ \Delta_{\mathbf{M}}(\bm{w},\bm{w}^k,\bm{w}^{k+1}),\label{Theory_Contractive_3}
	\end{array}
	\end{equation}
	where $\Delta_{\mathbf{M}}(\bm{w},\bm{w}^k,\bm{w}^{k+1}) = \frac{1}{2}(\Lambda^{k+1}-\Lambda^k+\Lambda^{k,k+1})$ and $\Delta_{\bar{\mathbf{W}}}$ has the similar form with $\Delta_{2}$ and $\Delta_{\mathbf{M}}$. 
	The proof is complete.
\end{proof}

\subsection{Proof of the Proposition~\ref{prop:sum_Corssing_Term}}% Proposition 3.
\begin{proof}
	Recall the result \eqref{eq:cretirion_relation}. By mathematical induction, for all $k \ge 1$, we have
	\begin{equation}\label{eq:Basic_equation_3}
	\vspace{-0.3cm}
	\| \mathbf{e}_k (\hat{\mathbf{x}}^{k+1}) \|_2 \!\le\! \sum\limits_{i=0}^{k-1} \eta^{k - i} \gamma \Lambda^{i,i+1}_{1/2} + \eta^{k} \|  \mathbf{e}_0 (\hat{\mathbf{x}}^{1}) \|_2.
	\end{equation}
	With $\sum\limits_{k=1}^{K} \bm{q}_k(\mathbf{x})^\top  \mathbf{e}_k (\hat{\mathbf{x}}^{k+1})
	\le \sum\limits_{k=1}^{K} \| \bm{q}_k(\mathbf{x})\|_2  \|\mathbf{e}_k (\hat{\mathbf{x}}^{k+1}) \|_2$ and Eq.~\eqref{eq:Basic_equation_3}, for all $\mathbf{x}\in \mathbb{R}^n$, $\mu>0$ and $K > 1$, then we have%yields 
	\begin{equation}
	\begin{array}{l}
	\sum\limits_{k=1}^{K} \bm{q}_k(\mathbf{x})^\top  \mathbf{e}_k (\hat{\mathbf{x}}^{k+1})\\
	\le \frac{\mu}{2} \sum\limits^{K}_{k=2}\sum\limits_{i=1}^{k-1} \eta^{k-i} \| \bm{q}_k(\mathbf{x})\|_2^2 + \frac{\mu}{2} \sum\limits^{K}_{k=1} \eta^{k} \| \bm{q}_k(\mathbf{x})\|_2^2 \nonumber\\
	+ \frac{1}{2\mu} \sum\limits^{K}_{k=2} \sum\limits_{i=1}^{k-1} \eta^{k-i}\gamma^2\Lambda^{i,i+1}
	\!+\! \frac{1}{2\mu} \sum\limits^{K}_{k=1} \eta^{k} [ \|  \mathbf{e}_0 (\hat{\mathbf{x}}^{1}) \|_2 \!+\! \gamma \Lambda^{i,i+1}_{1/2} ]^2\nonumber\\
	= \frac{\mu}{2} \sum\limits^{K}_{k=1}\sum\limits_{i=0}^{k-1} \eta^{k-i} \| \bm{q}_k(\mathbf{x})\|_2^2
	+\frac{1}{2\mu} \sum\limits^{K}_{k=2} \sum\limits_{i=1}^{k-1} \eta^{k-i}\gamma^2\Lambda^{i,i+1}  \nonumber\\
	+\frac{1}{2\mu} \sum\limits^{K}_{k=1} \eta^{k} \left( \|  \mathbf{e}_0 (\hat{\mathbf{x}}^{1}) \|_2 + \gamma \Lambda^{0,1}_{1/2} \right)^2.\nonumber
	\end{array}
	\end{equation}
	Furthermore, for all $\mathbf{x} \in \mathbb{R}^n$, $K > 1$ and $\eta\in(0,1)$, we have
	$
	\sum\limits_{k=1}^{K}\sum\limits_{i=0}^{k-1} \eta^{k-i} \| \bm{q}_k(\mathbf{x})\|_2^2 =
	\sum\limits_{k=1}^{K} \frac{\eta - \eta^{k+1}}{1-\eta}  \| \bm{q}_k(\mathbf{x})\|_2^2
	$
	and
	\begin{equation}
	\begin{array}{l}
	\sum\limits_{k = 2}^{K} \sum\limits_{i=1}^{k-1} \eta^{k-i}\gamma^2\Lambda^{i,i+1}
	= \sum\limits_{i=1}^{K-1}\sum\limits_{k=i+1}^{K} \eta^{k-i}\gamma^2\Lambda^{i,i+1}\nonumber\\
	= \sum\limits_{i=1}^{K-1}\frac{\eta - \eta^{K-i+1}}{1 - \eta} \gamma^2\Lambda^{i,i+1}.\label{Basic_equation_5}
	\end{array}
	\end{equation}
	Combining the above equalities and inequalities, we obtain
	$$
	\begin{array}{l}
	\sum\limits_{k=1}^{K} \bm{q}_k(\mathbf{x})^\top  \mathbf{e}_k (\hat{\mathbf{x}}^{k+1}) \\
	= \frac{\mu}{2} \sum\limits^{K}_{k=1} \frac{\eta - \eta^{k+1}}{1-\eta}  \| \bm{q}_k(\mathbf{x})\|_2^2
	+ \frac{1}{2\mu} \sum\limits^{K-1}_{i=1}\frac{\eta - \eta^{K-i+1}}{1 - \eta} \gamma^2\Lambda^{i,i+1} \\
	+ \frac{1}{2\mu} \frac{ \eta - \eta^{K+1} }{1 - \eta} \left( \| \mathbf{e}_0 (\hat{\mathbf{x}}^{1}) \|_2 + \gamma \Lambda^{0,1}_{1/2} \right)^2 \\
	\le \frac{\mu}{2} \sum\limits^{K}_{k=1} \frac{\eta}{1-\eta}  \| \bm{q}_k(\mathbf{x})\|_2^2 + \frac{1}{2\mu} \sum\limits^{K-1}_{i=1}\frac{\eta}{1 - \eta} \gamma^2\Lambda^{i,i+1} \\
	\quad  + \frac{1}{2\mu} \frac{ \eta }{1 - \eta} \left( \| \mathbf{e}_0 (\hat{\mathbf{x}}^{1}) \|_2 + \gamma \Lambda^{0,1}_{1/2} \right)^2, 
	\end{array}	$$
	which implies the conclusion \eqref{eq:re_Var_Inequality}. 
	The proof is complete.
\end{proof}

\subsection{Proof of the Theorem~\ref{thm:Convergence}}
\begin{proof}
	First, we define $\Gamma=( \bm{w} - \bar{\bm{w}}^k )^\top  ( \mathbf{F} (\bm{w}) -\mathbf{F} (\bar{\bm{w}}^k))$ and recall the definition of $\mathbf{F}(\bm{w})$ in \eqref{eq:notation}. We have
	\begin{equation}\label{Basic_equation_8}
	\begin{array}{l}
	\Gamma
	= ( \mathbf{x}- \mathbf{x}^{k+1} )^\top  \mathbf{Q}^\top  ( \nabla l( \mathbf{Q} \mathbf{x} ) - \nabla l(\mathbf{Q} \mathbf{x}^{k+1}))\\
	\quad \overset{\eqref{eq:StrongConvex}}{\ge} \theta \| \bm{q}_k(\mathbf{x}) \|_2^2.
	\end{array}
	\end{equation}
	Then, using the results~\eqref{eq:re_Var_Inequality} and Propositions~\ref{prop:re_Var_Inequality} and~\ref{prop:sum_Corssing_Term}, respectively, we obtain	
	\begin{equation}\label{eq:Theory_Contractive_5}
	\begin{array}{l}
	\!\!\sum\limits^{K}_{k=1} V(\bar{\mathbf{y}}^k,\mathbf{y},\bar{\bm{w}}^k,\bm{w})\\
	\!\!= \sum\limits^{K}_{k=1}\left\{g( \bar{\mathbf{y}}^k ) \!-\! g( \mathbf{y} )
	\!+\! ( \bar{\bm{w}}^k \!-\!\bm{w})^\top  \mathbf{F}( \bar{\bm{w}}^k)
	-\Gamma\right\}\\
	\!\!\overset{\eqref{eq:re_Var_Inequality}}{\le}\! \frac{1}{2} \left( \Lambda^{1} \!-\! \Lambda^{K\!+\!1} \right) \!+\!\sum\limits^{K}_{k=1} \left\{\bm{q}_k(x)^\top  \mathbf{e}_k (\hat{\mathbf{x}}^{k+1}) \!-\! \Gamma\right\} \!-\! \sum\limits_{k=1}^{K} \frac{1}{2}\Lambda^{k,k+1}\\
	\!\!\overset{\eqref{eq:sum_Crossing_Term}\eqref{Basic_equation_8}}{\le} \frac{1}{2} ( \Lambda^1 \!-\! \Lambda^{K+1}) 
	\!+\!\sum\limits^{K}_{k=1} ( \frac{\mu}{2} \frac{\eta}{1 - \eta} \!-\! \theta ) \| \bm{q}_k(\mathbf{x})\|_2^2  \\
	\!\!+\sum\limits^{K-1}_{k=1} \frac{1}{2} ( \frac{\eta}{1 - \eta} \frac{\gamma^2}{\mu} \!-\! 1 ) \Lambda^{k,k+1} \!-\! \frac{1}{2} \Lambda^{K,K\!+\!1} \\
	\!\!+\frac{1}{2\mu} \frac{\eta}{1 - \eta} ( \| \mathbf{e}_0 (\hat{\mathbf{x}}^{1}) \|_2 \!+\! \gamma \Lambda^{0,1}_{1/2} )^2.
	\end{array}
	\end{equation}
	For any given $\bm{w}^{\ast}\in\mathbf{\Omega}^{\ast}$, we have
	$V(\bar{\mathbf{y}}^k,\mathbf{y}^{\ast},\bar{\bm{w}}^k,\bm{w}^{\ast})\ge 0$, $\forall k$.
	Setting $\bm{w} = \bm{w}^{\ast}$ in \eqref{eq:Theory_Contractive_5}, together with the above property, for any $K > 1$, we have
	\begin{equation*}
	\begin{array}{l}
	\!\!\sum\limits^{K}_{k=1} (\theta- \frac{\mu}{2} \frac{\eta}{1 - \eta} ) \|  \bm{q}_k(\mathbf{x}^{\ast})\|_2^2 
	\!+\!\sum\limits^{K-1}_{k=1} ( \frac{1}{2} - \frac{\gamma^2}{2\mu} \frac{\eta}{1 - \eta} ) \Lambda^{k,k+1}\\
	\!\!\le\! \frac{1}{2} \Lambda^{1,\ast}\!+\!\frac{1}{2\mu} \frac{\eta}{1 - \eta} ( \| \mathbf{e}_0 (\hat{\mathbf{x}}^{1}) \|_2 \!+\! \gamma \Lambda^{0,1}_{1/2} )^2
	\!-\!\frac{1}{2} \Lambda^{K+1,\ast} \!-\! \frac{1}{2} \Lambda^{K,K+1}.
	\end{array}
	\end{equation*}
	Recall that the parameter $\eta$ controlling the accuracy in~\eqref{eq:ek-condition} is restricted by the condition~\eqref{eq:etaCriterion}. Hence, it follows from the definition of $\gamma$ that
	$
	0<\frac{ \gamma^2 \eta^2}{2\theta(1 - \eta)^2}=(  \frac{\eta}{2\theta(1 - \eta)} ) ( \frac{\gamma^2 \eta}{1 - \eta} ) < 1
	$
	and obviously there exists a $\mu>0$ such that
	$
	0<\frac{\mu}{2\theta}\frac{\eta}{1-\eta} < 1 \quad \hbox{and} \quad 0<\frac{\gamma^2}{\mu}\frac{\eta}{1-\eta} < 1.
	$
	As $k\rightarrow\infty$ we conclude that 
	$
	\|  \bm{q}_k(\mathbf{x}^{\ast})\|_2\rightarrow 0,\Lambda^{k,k+1}_{1/2}\rightarrow 0,\;\Lambda^{K+1,\ast}_{1/2} < \infty.
	$
	Furthermore, for any $\varepsilon > 0$, there exists $k_0$ such that for all $k\ge k_0$, we have
	$
	\Lambda^{k,k+1}_{1/2}\le \varepsilon \quad \hbox{and}\quad \eta^k \le \varepsilon.
	$
	For all $k > 2k_0$, it follows from \eqref{eq:Basic_equation_3} that
	\begin{equation}
	\begin{array}{l}
	\| \mathbf{e}_k (\hat{\mathbf{x}}^{k+1}) \|_2
	\le \sum\limits_{i=0}^{k-1} \eta^{k - i} \gamma \Lambda^{i,i+1}_{1/2} + \eta^{k} \| \mathbf{e}_0 (\hat{\mathbf{x}}^{1}) \|_2\\
	=\sum\limits_{i=0}^{k_0-1} \eta^{k - i} \gamma \Lambda^{i,i+1}_{1/2}
	+ \sum\limits_{k_0}^{k-1} \eta^{k - i} \gamma \Lambda^{i,i+1}_{1/2} + \eta^{k} \| \mathbf{e}_0 (\hat{\mathbf{x}}^{1}) \|_2\\
	\le \left( \max\limits_{\mathbf{0}\le i\le k_0-1}\left\{ \Lambda^{i,i+1}_{1/2} \right\}\gamma \sum\limits_{i=0}^{k_0-1} \eta^{k-k_0-i} \right) \cdot \eta^{k_0} \\
	+ \eta^{k} \| \mathbf{e}_0 (\hat{\mathbf{x}}^{1}) \|_2 + \left( \sum\limits_{k_0}^{k-1} \eta^{k - i} \gamma \right)\cdot \max\limits_{k_0\le i\le k-1}\left\{ \Lambda^{i,i+1}_{1/2} \right\}\nonumber\\
	\le \left[ \left( \max\limits_{\mathbf{0}\le i\le k_0-1}\left\{ \Lambda^{i,i+1}_{1/2} \right\}\gamma \sum\limits_{i=0}^{k_0-1} \eta^{k-k_0-i} \right)\right.\\
	\left.+\sum\limits_{k_0}^{k-1} \eta^{k - i} \gamma  + \| \mathbf{e}_0 (\hat{\mathbf{x}}^{1}) \|_2 \right]\cdot \varepsilon,
	\end{array}	\end{equation}
	which implies that
	$
	\| \mathbf{e}_k (\hat{\mathbf{x}}^{k+1}) \|_2 \overset{k\rightarrow \infty}{\longrightarrow} 0.
	$
	Moreover, note that $\| \mathbf{B}( \mathbf{y}^{k} - \mathbf{y}^{k+1} )\|_2 \overset{k\rightarrow \infty}{\longrightarrow} 0$ can be obtained by the fact $\| \bm{w}^{k} - \bm{w}^{k+1} \|_{\mathbf{M}} \overset{k\rightarrow \infty}{\longrightarrow} 0$. The first assertion is proved.
	
	Now we prove the second assertion. For the first part: $\|\mathbf{A} \mathbf{x}^{k+1} + \mathbf{B} \mathbf{y}^{k+1} - \mathbf{c}\|_2 \overset{k\rightarrow \infty}{\longrightarrow} 0$, it follows immediately from the facts $\|\mathbf{A} \mathbf{x}^{k+1} + \mathbf{B} \mathbf{y}^{k+1} - \mathbf{c}\|_2=\frac{1}{\beta}\|\bm{\lambda }^{k} - \bm{\lambda }^{k+1} \|_2$  and $\Lambda^{k,k+1}_{1/2} \overset{k\rightarrow \infty}{\longrightarrow} 0$. Note that the optimality conditions of the $\mathbf{y}$-subproblem at the ($k$+$1$)-th iteration and a solution point $\mathbf{y}^{\ast}$ can be respectively written as
	$$
	\left\{\begin{array}{l}
	g(\mathbf{y}) - g ( \mathbf{y}^{k+1}) + ( \mathbf{y}- \mathbf{y}^{k+1})^\top  ( - \mathbf{B}^\top  \bm{\lambda }^{k+1}) \ge \mathbf{0},\\
	g(\mathbf{y}) - g( \mathbf{y}^{\ast} ) + ( \mathbf{y}- \mathbf{y}^{\ast} )^\top  ( - \mathbf{B}^\top  \bm{\lambda }^{\ast}) \ge \mathbf{0}.
	\end{array}\right.	$$
	Accordingly, taking $\mathbf{y} = \mathbf{y}^{\ast}$ and $\mathbf{y} = \mathbf{y}^{k+1}$ respectively in the above inequalities, we have
	\begin{equation}\label{Basic_equation_10}
	( \mathbf{y}^{k+1} \!-\! \mathbf{y}^{\ast} )^\top  \mathbf{B}^\top  \bm{\lambda }^{\ast} \!\le\! g ( \mathbf{y}^{k+1} ) - g ( \mathbf{y}^{\ast} )\!\le\! ( \mathbf{y}^{k+1} \!-\! \mathbf{y}^{\ast} )^\top  \mathbf{B}^\top  \bm{\lambda }^{k+1}.
	\end{equation}
	The same technique can also be applied to the $\mathbf{x}$-subproblem and a solution point $x^{\ast}$. Additionally, using the convexity of $l$, we have
	\begin{equation}\label{Basic_equation_11}
	\begin{array}{l}
	( \mathbf{x}^{k+1} - \mathbf{x}^{\ast} )^\top  \mathbf{A}^\top  \bm{\lambda }^{\ast}
	\le l (\mathbf{Q}\mathbf{x}^{k+1} ) - l (\mathbf{Q} \mathbf{x}^{\ast} )\\
	\le ( \mathbf{x}^{k+1} - \mathbf{x}^{\ast} )^\top  \mathbf{Q}^\top \nabla l( \mathbf{Q} \mathbf{x}^{k+1} )\\
	= ( \mathbf{x}^{k+1} - \mathbf{x}^{\ast} )^\top  [ \mathbf{A}^\top  ( \bm{\lambda }^k - \beta ( \mathbf{A}\mathbf{x}^{k+1} + \mathbf{B} \mathbf{y}^k - \mathbf{c} ) )\\
	-\bar{\mathbf{W}}(\mathbf{x}^{k+1}-\mathbf{x}^k) + \mathbf{Q}^\top  \mathbf{e}_{k} (\hat{\mathbf{x}}^{k+1})  ]:=\bm{\Upsilon}.
	\end{array}
	\end{equation}
	Then, summarizing (\ref{Basic_equation_10}) and (\ref{Basic_equation_11}), we obtain
	\begin{equation}\label{Basic_equation_12}
	\begin{array}{l}
	\!\!\frac{1}{\beta} ( \bm{\lambda }^{k} - \bm{\lambda }^{k+1} )^\top  \bm{\lambda }^{\ast}\\
	\!\!= ( \mathbf{x}^{k+1} - \mathbf{x}^{\ast} )^\top  \mathbf{A}^\top  \bm{\lambda }^{\ast} + ( \mathbf{y}^{k+1} - \mathbf{y}^{\ast} )^\top  \mathbf{B}^\top  \bm{\lambda }^{\ast}\\
	\!\!\le  [ f ( \mathbf{x}^{k+1} ) + g ( \mathbf{y}^{k+1} )] - [ g ( \mathbf{y}^{\ast} ) + f ( \mathbf{x}^{\ast} )]\\
	\!\!\leq \bm{\Upsilon} +  ( \mathbf{y}^{k+1} - \mathbf{y}^{\ast} )^\top  \mathbf{B}^\top  \bm{\lambda }^{k+1} \\
	\!\!\leq \frac{1}{\beta}( \bm{\lambda }^k - \bm{\lambda }^{k+1} )^\top  \bm{\lambda }^{k+1} + \beta( \mathbf{x}^{k+1} - \mathbf{x}^{\ast} )^\top  \mathbf{A}^\top  \mathbf{B} ( \mathbf{y}^{k+1} - \mathbf{y}^k )\\
	\!\!-( \mathbf{x}^{k+1} - \mathbf{x}^{\ast} )^\top  \bar{\mathbf{W}}(\mathbf{x}^{k+1}-\mathbf{x}^k)+ ( \mathbf{x}^{k+1} - \mathbf{x}^{\ast} )^\top  \mathbf{Q}^\top  \mathbf{e}_{k} (\hat{\mathbf{x}}^{k+1}).
	\end{array}
	\end{equation}
	Since $\|\mathbf{e}_k (\hat{\mathbf{x}}^{k+1}) \|_2\rightarrow \mathbf{0},\ \| \mathbf{Q} \mathbf{x}^{k+1} - \mathbf{Q}\mathbf{x}^{\ast}\|_2 \rightarrow \mathbf{0},\ \Lambda^{k,k+1}_{1/2}\rightarrow \mathbf{0},\ \Lambda^{k+1,\ast}_{1/2} < \infty$, as well as
	$
	( \mathbf{x}^{k+1} - \mathbf{x}^{\ast} )^\top  \bar{\mathbf{W}}(\mathbf{x}^{k+1}-\mathbf{x}^k)
	\le \Lambda^{k+1,\ast}_{1/2} \Lambda^{k,k+1}_{1/2}\rightarrow \mathbf{0},
	$
	and
	\begin{equation*}
	\begin{array}{l}
	\mathbf{A}( \mathbf{x}^{k+1} - \mathbf{x}^{\ast} ) = \frac{1}{\beta}( \bm{\lambda }^k - \bm{\lambda }^{k+1} ) - \mathbf{B}( \mathbf{y}^{k+1} - \mathbf{y}^{\ast} )\\
	\Rightarrow\| \mathbf{A}( \mathbf{x}^{k+1} - \mathbf{x}^{\ast} ) \|_2^2 < \infty,
	\end{array}
	\end{equation*}
	both the left- and right-hand sides of (\ref{Basic_equation_12}) converge to zero. As  a result, we have
	$
	l (\mathbf{Q} \mathbf{x}^{k+1} ) + g ( \mathbf{y}^{k+1} ) \overset{k\rightarrow \infty}{\longrightarrow} l(\mathbf{Q} \mathbf{x}^\ast ) + g ( \mathbf{y}^{\ast} ),
	$which is the second assertion of this theorem. The proof is complete.
\end{proof}

\subsection{Proof of the Corollary~\ref{corollary}}

\begin{proof}
	The proof of this Corollary can be conducted following the same road-maps developed in \cite{yue2018implementing}. Indeed, from inequality~\eqref{eq:Theory_Contractive_5}, it is easy to show the upper bound of $\min\limits_{1\le k\le K}\{\Lambda^{k,k+1}\}$  is in order of  ${\cal{O}}(\frac{1}{k})$. With the help of inequality~\eqref{eq:Basic_equation_3}, one can also obtain the upper bound of $\min\limits_{1\le k\le K}\{\|\mathbf{e}_k(\hat{\mathbf{x}}^{k+1})\|_2^2\}$ with the same order. The details of the proof are similar to \cite{yue2018implementing} and we do not repeat it here.
\end{proof}

% use section* for acknowledgment
\ifCLASSOPTIONcompsoc
% The Computer Society usually uses the plural form
\section*{Acknowledgments}
\else
% regular IEEE prefers the singular form
\section*{Acknowledgment}
\fi

This work is partially supported by the National Key R\&D Program of China (2020YFB1313503), the National Natural Science Foundation of China (Nos. 61922019, 61672125, 61733002, 61432003, 61632019, 11971220), Shenzhen Science and Technology Program (No. RCYX20200714114700072), the Stable Support Plan Program of Shenzhen Natural Science Fund (No. 20200925152128002), Guangdong Basic and Applied Basic Research Foundation 2019A1515011152, and the Fundamental Research Funds for the Central Universities.

%Research of this author was supported by National Science Foundation of China 11971220, by Shenzhen Science and Technology Program (No. RCYX20200714114700072), by the Stable Support Plan Program of Shenzhen Natural Science Fund (No. 20200925152128002), and by Guangdong Basic and Applied Basic Research Foundation 2019A1515011152.

% Can use something like this to put references on a page
% by themselves when using endfloat and the captionsoff option.
\ifCLASSOPTIONcaptionsoff
\newpage
\fi

% trigger a \newpage just before the given reference
% number - used to balance the columns on the last page
% adjust value as needed - may need to be readjusted if
% the document is modified later
%\IEEEtriggeratref{8}
% The "triggered" command can be changed if desired:
%\IEEEtriggercmd{\enlargethispage{-5in}}

% references section

% can use a bibliography generated by BibTeX as a .bbl file
% BibTeX documentation can be easily obtained at:
% http://mirror.ctan.org/biblio/bibtex/contrib/doc/
% The IEEEtran BibTeX style support page is at:
% http://www.michaelshell.org/tex/ieeetran/bibtex/
\bibliographystyle{IEEEtran}
\bibliography{reference}

\begin{IEEEbiography}[{\includegraphics[width=1in,height=1.25in,clip,keepaspectratio]{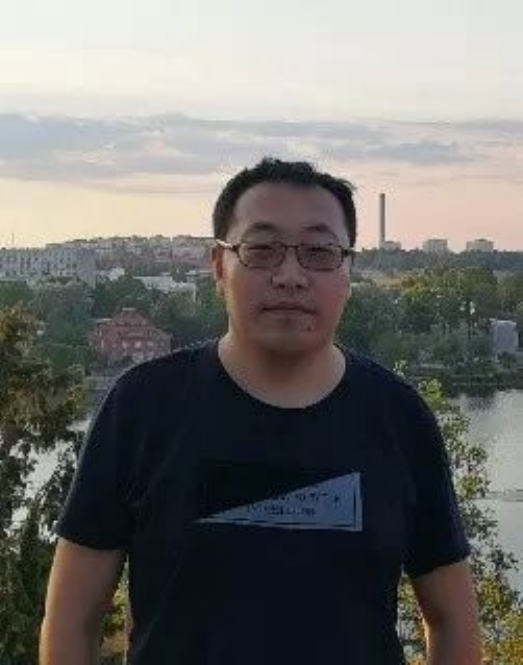}}]{Risheng Liu} received the B.S. and Ph.D. degrees both in mathematics from the Dalian University of Technology in 2007 and 2012, respectively. He was a visiting scholar in the Robotic Institute of Carnegie Mellon University from 2010 to 2012. He served as Hong Kong Scholar Research Fellow at the Hong Kong Polytechnic University from 2016 to 2017. He is currently a professor with DUT-RU International School of Information Science \& Engineering, Dalian University of Technology. He was awarded the ``Outstanding Youth Science Foundation'' of the National Natural Science Foundation of China. His research interests include machine learning, optimization, computer vision and multimedia. He was a co-recipient of the IEEE ICME Best Student Paper Award in both 2014 and 2015. His two papers were also selected as Finalist of the Best Paper Award in ICME 2017. He is a member of the IEEE and ACM.
\end{IEEEbiography}

\begin{IEEEbiography}[{\includegraphics[width=1in,height=1.25in,clip,keepaspectratio]{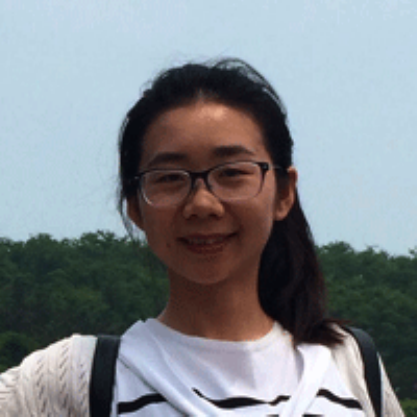}}]{Pan Mu} received the B.S. degree in Applied Mathematics from Henan University, China, in 2014, the M.S. degree in Operational Research and Cybernetics from Dalian University of Technology, China, in 2017. She is currently pursuing the PhD degree in Computational Mathematics at Dalian University of Technology, China. Her research interests include computer vision, machine learning and optimization. 
\end{IEEEbiography}
\begin{IEEEbiography}[{\includegraphics[width=1in,height=1.32in,clip,keepaspectratio]{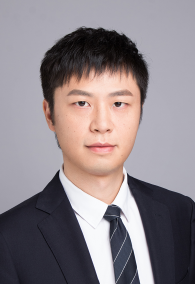}}]{Jin Zhang} received the B.A. degree in Journalism from the Dalian University of Technology in 2007. He pursued a degree in mathematics and received the M.S. degree in Operational Research and Cybernetics from the Dalian University of Technology, China, in 2010, and the PhD degree in Applied Mathematics from University of Victoria, Canada, in 2015. %He served as Hong Kong Scholar Research Fellow at the Hong Kong Baptist University, China, in 2015. 
	%He served as Research Assistant Professor at 
	After working in Hong Kong Baptist University for 3 years, he joined Southern University of Science and Technology as a tenure-track assistant professor in the Department of Mathematics. His broad research area is comprised of optimization, variational analysis and their applications in economics, engineering and data science.
\end{IEEEbiography}

% biography section
% 
% If you have an EPS/PDF photo (graphicx package needed) extra braces are
% needed around the contents of the optional argument to biography to prevent
% the LaTeX parser from getting confused when it sees the complicated
% \includegraphics command within an optional argument. (You could create
% your own custom macro containing the \includegraphics command to make things
% simpler here.)
%\begin{IEEEbiography}[{\includegraphics[width=1in,height=1.25in,clip,keepaspectratio]{mshell}}]{Michael Shell}
% or if you just want to reserve a space for a photo:

%\begin{IEEEbiography}{Michael Shell}
%Biography text here.
%\end{IEEEbiography}
%
%% if you will not have a photo at all:
%\begin{IEEEbiographynophoto}{John Doe}
%Biography text here.
%\end{IEEEbiographynophoto}
%
%% insert where needed to balance the two columns on the last page with
%% biographies
%%\newpage
%
%\begin{IEEEbiographynophoto}{Jane Doe}
%Biography text here.
%\end{IEEEbiographynophoto}

% You can push biographies down or up by placing
% a \vfill before or after them. The appropriate
% use of \vfill depends on what kind of text is
% on the last page and whether or not the columns
% are being equalized.

%\vfill

% Can be used to pull up biographies so that the bottom of the last one
% is flush with the other column.
%\enlargethispage{-5in}

% that's all folks
\end{document}